\documentclass[11pt, a4paper]{article}

\usepackage[a4paper]{geometry}
\AtBeginDocument{%
\setlength{\oddsidemargin}{\dimexpr(\paperwidth-\textwidth)/2-1in}%
\setlength{\evensidemargin}{\oddsidemargin}%
\setlength{\topmargin}{%
\dimexpr(\paperheight-\textheight)/2-\headheight-\headsep-1in}%
}

\usepackage[utf8]{inputenc}
\usepackage[english]{babel}

\usepackage{xspace}
\usepackage{amsmath,amssymb,mathtools,amsthm}\usepackage{lmodern}
\usepackage{bbm}
\usepackage{algorithm}
\usepackage[noend]{algpseudocode}
\usepackage{xcolor}
\usepackage{graphicx}
\renewcommand{\labelenumi}{(\alph{enumi})}
\renewcommand\theenumi\labelenumi

\usepackage{dsfont}
\usepackage{todonotes}
\usepackage{booktabs}
\usepackage{footnote}
\usepackage[shortcuts]{extdash}
\makesavenoteenv{tabular}
\makesavenoteenv{table}
\usepackage{bbm}
\usepackage{algorithm}
\usepackage[noend]{algpseudocode}
\usepackage{xcolor}
\usepackage{url}

\newtheorem{theorem}{Theorem}
\newtheorem{lemma}[theorem]{Lemma}

\newtheorem{definition}[theorem]{Definition}

\newcommand{\ooea}{(1+1)~EA\xspace}

\newcommand{\oplea}{(1+$\lambda$)~EA\xspace}
\newcommand{\lea}{\oplea}
\newcommand{\mpoea}{($\mu$+1)~EA\xspace}
\newcommand{\mplea}{($\mu$+$\lambda$)~EA\xspace}

\newcommand{\oclea}{(1,$\lambda$)~EA\xspace}

\newcommand{\om}{\textsc{OneMax}\xspace}
\newcommand{\onemax}{\om}
\newcommand{\lo}{\textsc{LeadingOnes}\xspace}
\newcommand{\leadingones}{\lo}
\newcommand{\R}{\ensuremath{\mathbb{R}}}
\newcommand{\N}{\ensuremath{\mathbb{N}}}
\newcommand{\Z}{\ensuremath{\mathbb{Z}}}
\DeclareMathOperator{\Prob}{Pr}
\DeclareMathOperator{\init}{init}
\DeclareMathOperator{\Bin}{Bin}
\newcommand{\xmin}{x_{\mathrm{min}}}
\newcommand{\xmax}{x_{\mathrm{max}}}
\newcommand{\rmax}{r_{\mathrm{max}}}

\newcommand{\eg}{e.\,g.\xspace}
\newcommand{\Var}{\mathrm{Var}\xspace} 
\newcommand{\eps}{\varepsilon}

\newcommand{\E}{E}
\newenvironment{proofof}[1]{\begin{proof}[of #1]}{\end{proof}}

\title{Runtime Analysis for Self-adaptive Mutation Rates\thanks{Extended version of a paper appearing at the \emph{Genetic and Evolutionary Computation Conference} 2018~\cite{DoerrWY18}. This version contains all proofs, whereas most of them for reasons of space did not fit into the conference version. In this version, the main result is valid for all $\lambda \ge C \ln(n)$, $C$ a sufficiently large constant, whereas the conference version needed $\lambda \ge (\ln n)^{1+\eps}$ for an arbitrary $\eps>0$.}}

\author{Benjamin Doerr\\ \raisebox{0mm}[0mm][0mm]{\'E}cole Polytechnique\\ CNRS\\
Laboratoire d'Informatique (LIX)\\Palaiseau, France
\and
Carsten Witt\\
DTU Compute\\Technical University of Denmark\\Kgs.\ Lyngby, Denmark
\and
Jing Yang\\ {\'E}cole Polytechnique\\ CNRS\\
Laboratoire d'Informatique (LIX)\\Palaiseau, France}
\begin{document}
{\sloppy   
\maketitle
\setcounter{secnumdepth}{2}
\begin{abstract}
We propose and analyze a self-adaptive version of the $(1,\lambda)$ evolutionary algorithm in which the current mutation rate is part of the individual and thus also subject to mutation. A rigorous runtime analysis on the \onemax benchmark function reveals that a simple local mutation scheme for the rate leads to an expected optimization time (number of fitness evaluations) of $O(n\lambda/\log\lambda+n\log n)$ when $\lambda$ is at least $C \ln n$ for some constant $C > 0$. For all values of $\lambda \ge C \ln n$, this performance is asymptotically best possible among  all $\lambda$-parallel mutation-based unbiased black-box algorithms. 

Our result shows that self-adaptation in evolutionary computation can find complex optimal parameter settings on the fly. At the same time, it proves that a relatively complicated self-adjusting scheme for the mutation rate proposed by Doerr, Gie{\ss}en, Witt, and Yang~(GECCO~2017) can be replaced by our simple endogenous scheme. 

On the technical side, the paper contributes new tools for the analysis of two-dimensional drift processes arising in the analysis of dynamic parameter choices in  EAs, including bounds on occupation probabilities in processes with non-constant drift.
\end{abstract}

%\onecolumn
\section{Introduction}

Evolutionary algorithms are a class of heuristic algorithms that can be applied to solve optimization problems if no problem-specific algorithm is available. For example, 
this may be the case if the structure of the underlying problem is poorly understood or one is faced with a so-called 
black-box scenario, in which the quality of a solution can only be determined by calling an  
implementation of the objective function. This implementation may be implicitly given by, \eg, the outcome of a simulation without 
revealing structural relationships between search point and function value.

An approach to understand the working principles of evolutionary algorithms is to analyze the underlying stochastic process 
and its first hitting time of the set of optimal or approximate solutions. The runtime analysis community in 
evolutionary computation (see, \eg, \cite{AugerD11,Jansen13,NeumannW10} for an introduction to the subject) 
follows this approach by partly using methods known from the analysis of classical randomized 
algorithms and, more recently and increasingly often, 
using and adapting tools from the theory of stochastic processes to obtain bounds 
on the hitting time of optimal solutions for different classes of 
evolutionary algorithms and problems. Such bounds will typically 
depend on problem size, problem type, evolutionary algorithm and choice of the parameters that these 
heuristic algorithms come with.

One of the core difficulties when using evolutionary algorithms is in fact 
finding suitable values for its parameters. It is well known and supported by ample experimental and some 
theoretical evidence that already small changes of the parameters can have a crucial influence on the efficiency of the algorithm. 

One elegant way to overcome this difficulty, and in addition the difficulty that the optimal parameter values may change during a run of the algorithm, is to let the algorithm optimize the parameters \emph{on the fly}. Formally speaking, this is an even more complicated task, because  instead of a single good parameter value now a suitable functional dependence of the parameter on the search history needs to be provided. Fortunately, a number of natural heuristics like the $1/5$-th rule have proven to be effective in certain cases. In a sense, these are all \emph{exogenous} parameter control mechanisms which are added to the evolutionary system.

An even more elegant way is to incorporate the parameter control mechanism into the evolutionary process, that is, to attach the parameter value to the individual, to modify it via (extended) variation operators, and to use the fitness-based selection mechanisms of the algorithm to ensure that good parameter values become dominant in the population. This \emph{self-adaptation} of the parameter values has two main advantages: (i)~It is generic, that is, the adaptation mechanism is provided by the algorithm, only the representation of the parameter in the individual and the extension of the variation operators has to be provided by the user. (ii)~It allows to re-use existing algorithms and much of the existing code. 

Despite these advantages, self-adaptation is not used a lot in discrete evolutionary optimization. From the theory side, some advice exists how to set up such a self-adaptive system, but a real proof for its usefulness is still missing. This is the point we aim to make some progress on. 

\subsection{Our Results} 

The main result of this work is that we propose a version of the $(1,\lambda)$ evolutionary algorithm (EA) with a natural self-adaptive choice of the mutation rate. For $\lambda \ge C \ln n$, $C$ a sufficiently large constant, we prove that it optimizes the classic \onemax benchmark problem in a runtime that is asymptotically optimal among all $\lambda$-parallel black-box optimization algorithms and that is better than the known runtimes of the \oclea and the \lea for all static choices of the mutation rate. Compared to the (also asymptotically optimal) \lea with fitness-dependent mutation rate of Badkobeh, Lehre, and Sudholt~\cite{BadkobehLS14} and the \lea with self-adjusting (exogenous) mutation rate of Doerr, Gie\ss en, Witt, and Yang~\cite{DoerrGWY17} the good news of our result is that this optimal runtime could be obtained in a generic manner. Note that both the fitness-dependent mutation rate of~\cite{BadkobehLS14} and the self-adjusting rate of~\cite{DoerrGWY17} with its mix of random and greedy rate adjustments would have been hard to find without a deeper understanding of the mathematics of these algorithms. 

Not surprisingly, the proof of our main result has some similarity to the analysis of the self-adjusting \lea of~\cite{DoerrGWY17}. In particular, we also estimate the expected progress in one iteration and use variable drift analysis. Also, we need a careful probabilistic analysis of the progress obtained from different mutation rates to estimate which rate is encoded in the new parent individual (unfortunately, we cannot reuse the analysis of~\cite{DoerrGWY17} since it is not always strong enough for our purposes). The reason, and this is also the main technical challenge in this work, is that the \oclea can lose fitness in one iteration. This happens almost surely when the mutation rate is too high. For this reason, we need to argue more carefully that such events do not happen regularly. To do so, among several new arguments, we also need a stronger version of the occupation probability result~\cite[Theorem~7]{KotzingLissWittFOGA15} since (i)~we need sharper probability estimates for the case that movements away from the target are highly unlikely and (ii)~for our process, the changes per time step cannot be bounded by a small constant. We expect our new results (Lemma~\ref{lem:occupation} and~\ref{lem:occsimple}) to find other applications in the theory of evolutionary algorithms in the future. Note that for the \lea, an excursion into unfavorable rate regions  is less a problem as long as one can show that the mutation rate returns into the good region after a reasonable time. The fact that the \oclea can lose fitness also makes it more difficult to cut the analysis into regimes defined by fitness levels since it is now possible that the EA returns into a previous regime.

In this work, we also gained two insights which might be useful in the design of future self-adaptive algorithms. 

\emph{Need for non-elitism:} Given the previous works, it would be natural to try a self-adaptive version of the \lea. However, this is risky. While the self-adjusting EA of~\cite{DoerrGWY17} copes well with the situation that the current mutation rate is far from the ideal one and then provably quickly changes the rate to an efficient setting, a self-adaptive algorithm cannot do so. Since the mutation rate is encoded in the individual, a change of the rate can only occur if an offspring is accepted. For an elitist algorithm like the \lea, this is only possible when an offspring is generated that is good enough to compete with the parent(s). Consequently, if the parent individual in a self-adaptive \lea has a high fitness, but a detrimental (that is, large) mutation rate, then the algorithm is stuck with this individual for a long time. Already for the simple \onemax function, such a situation can lead to an exponential runtime. 

Needless to say, when using a comma strategy we have to choose $\lambda$ sufficiently large to avoid losing the current-best solution too quickly. This phenomenon has been observed earlier, e.g., in~\cite{RoweS14} it is shown that $\lambda \ge (1-o(1)) \log_{(e-1)/e}(n)$ is necessary for the \oclea with mutation rate $1/n$ to have a polynomial runtime on any function with unique optimum. We shall not specify a precise leading constant for our setting, but also require that $\lambda \ge C \ln(n)$ for a sufficiently large constant $C$.

\emph{Tie-breaking towards lower mutation rates:} To prove our result, we need that the algorithm in case of many offspring of equal fitness prefers those with the smaller mutation rate. Given that the usual recommendation for the mutation rate is small, namely $\frac 1n$, and that it is well-known that large rates can be very detrimental, it is natural to prefer smaller rates in case of ties (where, loosely speaking, the offspring population gives no hint which rate is preferable). 

This choice is similar to the classic tie-breaking rule of preferring offspring over parents in case of equal fitness. Here, again, the fitness indicates no preference, but the simple fact that one is maybe working already for quite some time with this parent suggest to rather prefer the new individual. %Without proof, we remark that without this tie-breaking rule we would need $\lambda = n^{\Omega(1)}$ to ensure a positive drift towards the optimum throughout the process.%\benj{can we point to a particular lemma where this occurs?}.%We do not give a formal proof for the absolute necessity of our tie-breaking rule, but we can show that without this rule, we observe a drift away from the optimum once we are mildly close to it, which makes it hard to believe that an efficient optimization is possible from this point on.

%\merk{(CW) For now, I just claim that we can prove this drift without the tie-breaking rule.}
%\merk{to be done. Can we just look at the event that we move 
% away in each iteration until we reach the bad region? I'd say yes...} 

\subsection{Previous Works} 

This being a theoretical paper, for reasons of space we shall mostly review the relevant theory literature, and also this with a certain brevity. For a broader account of previous works, we refer to the survey~\cite{KarafotiasHE15}. For a detailed description of the state of the art in theory of dynamic parameter choices, we refer to the survey~\cite{DoerrD18survey}. We note that the use of self-adaptation in genetic algorithms was proposed in the seminal paper~\cite{Back92} by B\"ack. Also, we completely disregard evolutionary optimization in continuous search spaces due to the very different nature of optimization there (visible, e.g., from the fact that dynamic parameter changes, including self-adaptive choices, are very common and in fact necessary to allow the algorithms to approach the optimum with arbitrary precision).

The theoretical analysis of dynamic parameter choices started slow. A first paper ~\cite{JansenW06} on this topic in 2006 demonstrated the theoretical superiority of dynamic parameter choices by giving an artificial example problem for which any static choice of the mutation rate leads to an exponential runtime, whereas a suitable time-dependent choice leads to a polynomial runtime. Four years later~\cite{BottcherDN10}, it was shown that a fitness-dependent choice of the mutation rate can give a constant-factor speed-up when optimizing the \leadingones benchmark function (see~\cite[Section~2.3]{Doerr18evocoparxiv} for a simplified proof giving a more general result). The first super-constant speed-up on a classic benchmark function obtained from a fitness-dependent parameter choice was shown in~\cite{DoerrDE13}, soon to be followed by the paper~\cite{BadkobehLS14} which is highly relevant for this work. In~\cite{BadkobehLS14}, the \oplea with fitness-dependent mutation rate was analyzed. For a slightly complicated fitness-dependent mutation rate, an optimization time of $O(n \lambda / \log \lambda + n \log n)$ was obtained. Also, it was shown that no $\lambda$-parallel mutation-based unbiased black-box algorithm can have an asymptotically better optimization time.

Around that time, several successful self-adjusting (``on-the-fly'') parameter choices were found and analyzed with mathematical means. In~\cite{LassigS11}, a success-based multiplicative update of the population size $\lambda$ in the \lea is proposed and it is shown that this can lead to a reduction of the parallel runtime. A multiplicative update inspired by the $1/5$-th success rule from evolution strategies automatically finds parameter settings~\cite{DoerrD15} leading to the same performance as the fitness-dependent choice in~\cite{DoerrDE13}. Similar multiplicative update rules have been used to control the mutation strength for multi-valued decision variables~\cite{DoerrDK18} and the time interval for which a selected heuristic is used in~\cite{DoerrLOW18}. A learning-based approach was used in~\cite{DoerrDY16ppsn} to automatically adjust the mutation strength and obtain the performance of the fitness-dependent choice of~\cite{DoerrDY16}. Again a different approach was proposed in~\cite{DoerrGWY17}, where the mutation rate for the \lea was determined on the fly by creating half the offspring with a smaller and half the offspring with a larger mutation rate than the value currently thought to be optimal. As new mutation rate, with probability $\frac12$ the rate which produced the best offspring was chosen, with probability $\frac12$ a random of the two rates was chosen. The three different exogenous approaches used in these works indicate that a generic approach towards self-adjusting parameter choices, such as self-adaptation, would ease the design of such algorithms significantly. 

Surprisingly, prior to this work only a single runtime analysis paper for self-adapting parameter choices appeared. In~\cite{DangL16ppsn}, Dang and Lehre show several positive and negative results on the performance of a simple class of self-adapting evolutionary algorithms having the choice between several mutation rates. Among them, they show that such an algorithm having the choice between an appropriate and a destructively high mutation rate can optimize the \leadingones benchmark function in the usual quadratic time, whereas the analogous algorithm using a random of the two mutation rates (and hence in half the cases the right rate) fails badly and needs an exponential time. As a second remarkable result, they give an example setting where any constant mutation rate leads to an exponential runtime, whereas the self-adapting algorithm succeeds in polynomial time. As for almost all such examples, also this one is slightly artificial and needs quite some assumptions, for example, that all $\lambda$ individuals are initialized with the 1-point local optimum. Nevertheless, this result makes clear that self-adaptation can outperform static parameter choices. In the light of this result, the main value of our results is showing that asymptotic runtime advantages from self-adaptation can also be obtained in less constructed examples (of course, at the price that the runtime gap is not exponential).

To complete the picture on previous work relevant to ours, we finally quickly describe what is known on the performance of most common mutation-based algorithms for the \onemax benchmark function. For the simple \ooea, the expected runtime of $\Theta(n \log n)$ was determined in~\cite{Muhlenbein92} (upper bound) and~\cite{DrosteJW02} (lower bound, this result was announced already 1998). For the \oplea with $\lambda \le n^{1-\eps}$, $\eps > 0$ a constant, an expected runtime (number of fitness evaluations) of \[\Theta\left(\frac{n \lambda \log\log \lambda}{\log \lambda} + n \log n\right)\] was shown in~\cite{JansenDW05,DoerrK15}. For the \mpoea with polynomially bounded $\mu$, the expected runtime is $\Theta(\mu n + n \log n)$, see~\cite{Witt06}. Finally, the expected runtime of the \mplea was recently~\cite{AntipovDFH18} determined as
\[
\Theta\bigg(\frac{n\log n}{\lambda}+\frac{n}{\lambda / \mu} +
\frac{n\log^+\log^+ \lambda/ \mu}{\log^+ \lambda / \mu}\bigg),\] where $\log^+
x := \max\{1, \log x\}$ for all $x > 0$.

The earliest runtime analysis of the \oclea with mutation rate $1/n$ on \onemax is due to Jägersküpper and 
Storch \cite{JSFOCI07}, who prove a phase transition from exponential to polynomial runtime in the regime $\lambda = \Theta(\log n)$, 
leaving a gap of at least 21 between the largest $\lambda$ in the exponential regime and the smallest in the polynomial regime. 
This result was improved by Rowe and Sudholt \cite{RoweS14}, who determined the phase transition point to be  
the above-mentioned function $\log_{(e-1)/e}(n)$, up to lower order terms. Jägersküpper and Storch \cite{JSFOCI07} 
also obtain a useful 
coupling result: if $\lambda \ge c\ln n$ for a sufficiently large constant
$c>0$, the stochastic behavior of the \oplea and \oclea with high probability are identical for a certain polynomial 
(with degree depending on~$c$)
number of steps, allowing the above-mentioned results about the \oplea to be transferred to the \oclea.  

\subsection{Techniques}

One of the technical difficulties in our analysis is that our self-adaptive \oclea can easily lose fitness when the rate parameter is set to an unsuitable value. For this reason, we cannot use the general approach of the analysis of the self-adjusting \oplea in~\cite{DoerrGWY17}, which separated the analysis of the rate and the fitness by, in very simple words, first waiting until the rate is in the desired range and then waiting for a fitness improvement (of course, taking care of the fact that the rate could leave the desired range). To analyze the joint process of fitness and rate with its intricate interactions, we in particular use drift analysis with a two-dimensional distance function, that is, we map (e.g., in Lemma~\ref{lem:nearmain}) the joint space of fitness and rate suitably into the non-negative integers in a way that the expected value of this mapping decreases in each iteration. This allows to use well-known drift theorems.

The use of two-dimensional potential functions is not new in the analysis of evolutionary algorithms. However, so far only very few analyses exist that use this technique with dynamic parameter values and among these results, we feel that ours, in particular, Lemma~\ref{lem:nearmain}, are relatively easy to use. Again in very simple words, the distance function $g$ defined in the proof of Lemma~\ref{lem:nearmain} is the fitness distance plus a pessimistic estimate for the fitness loss that could be caused from the current rate if this is maladjusted). We thus hope that this work eases future analyses of dynamic parameter choices by suggesting ways to measure suitably the progress in the joint space of solution quality and parameter value.

To allow the reader to compare our two-dimensional drift approach with existing works using similar arguments, we briefly review the main works that use two- or more-dimensional potential functions. Ignoring that the artificial fitness functions used in~\cite{DrosteJW02,DoerrJohannsenWinzenALGO12,DoerrG13algo,Witt13} could also be interpreted as $n$-dimensional potential functions, the possibly first explicit use of a two-dimensional potential function in the runtime analysis of randomized search heuristics can be found in~\cite[proof of Theorem~4]{Wegener05}, a work analyzing how simulated annealing and the Metropolis algorithm compute minimum spanning trees in a line of connected triangles. In such optimization processes, a solution candidate (which is a subset of the edge set of the graph) can have two undesirable properties. (i) The solution contains a \emph{complete triangle}, so one of these three edges has to be removed on the way to the optimal solution. (ii) The solution contains two edges of a triangle, but not the two with smallest weight. This case, called \emph{bad triangle}, is the less desirable one as here one edge of the solution has to be replaced by the missing edge and hence the status of two edges has to be changed. It turns out that a simple potential function can take care of these two issues, namely twice the number of bad triangles plus the number of complete triangles.

When analyzing non-trivial parent populations, then often it does not suffice to measure the quality of the current state via the maximum fitness in the population, but also the number of individuals having this best fitness has to be taken into account. This was first done in the analysis of the \mpoea in~\cite{Witt06}. Since in a run of this algorithm the population never worsens (in a strong sense), the progress could be analyzed conveniently via arguments similar to the fitness level method. Consequently, it was not necessary to define an explicit potential function. In a similar fashion, the $(N+N)$ EA~\cite{ChenHSCY09} and the \mplea~\cite{AntipovDFH18} were analyzed by regarding the maximum fitness and the number of individuals having this fitness.

In~\cite{LehreY12}, a vaguely similar approach was taken for non-elitist population-based algorithms. However, the fact that these algorithm may lose the current-best solution required a number of non-trivial modifications, most notably, (i)~that the potential is based on the maximum fitness such that at least a proportion of $\gamma$ of the individuals have at least this fitness (for a suitable constant $0 < \gamma <1$) instead of the maximum fitness among all individuals, and (ii)~that the arguments resembling the fitness level method had to be replaced by a true drift argument. This approach was extended in~\cite{DangL16algo} to give a general ``level-based'' runtime analysis result. A simplified version of this level theorem was recently given in~\cite{CorusDEL18}.

What comes closest to our work with respect to the use of two-dimensional potential functions is~\cite{DoerrDK18}, where a self-adjusting bit-wise mutation strength for functions defined not over bit strings, but over $\{0, \dots, r-1\}^n$ for some $r > 2$ is discussed. The potential function defined in~(6) in~\cite[Section~7]{DoerrDK18} is too complicated to be described here in detail, but it also follows the pattern used in this work, namely that the potential (to be minimized) is the sum of the fitness distance and a penalty for mutation strengths deviating from their currently ideal value. This potential function, however, does not admit an easy interpretation of the type ``fitness distance plus expected damage from improper mutation strength'' as in our work. Consequently, the proof that indeed the desired progress is obtained with respect to this potential function is a lengthy (more than 4 pages) case distinction. Apparently unaware of the conference version~\cite{DoerrDK16PPSN}, a similar approach, also with a slightly complicated potential function, was developed in~\cite{AkimotoAG18} to analyze the (1+1) ES with $1/5$ success rule.

A very general approach was recently published in~\cite{Rowe18}. When a process $X_0, X_1, \dots$ admits several distance functions $d_1, \dots, d_m$ such that, for all $i \in [1..m]$, the $i$-th distance satisfies $E[d_i(X_{t+1}) \mid X_t] \le A (d_1(X_t), \dots, d_m(X_t))^\top$ for a given matrix $A$, then under some natural conditions the first time until all distances are zero can be bounded in terms of a suitable eigenvalue of $A$. The assumptions on the distance functions and the matrix $A$ are non-trivial, but~\cite{Rowe18} provides a broad selection of applications of this method. For our problem, we would expect that this method can be employed as well, however, this would also need an insight similar to the main insight of our approach, namely that the expected new fitness can be estimated in a linear fashion from the current fitness and the distance of the current rate from the ideal value.

\subsection{Organization of This Work}

This paper is structured as follows. In Section~\ref{sec:algo}, we define the self-adaptive \oclea proposed in this work. In Section~\ref{sec:tools} we provide the technical tools needed on our analysis, among them two new results on occupation probabilities. Section~\ref{sec:main} presents the main theorem. Its proof 
considers two main regions of different fitness, which are dealt with in separate subsections. 
We finish with some conclusions.

%\section{Preliminaries}
%
%\label{sec:preliminaries}
\section{The \oclea With Self-Adapting Mutation Rate}\label{sec:algo}

We now define precisely the \oclea with self-adaptive mutation rate proposed in this work. This algorithm, formulated for the \emph{minimization}
of pseudo-boolean functions $f:\{0,1\}^n\to\R$, is stated in pseudocode in Algorithm~\ref{alg:onelambda}.

To encode the mutation rate into the individual, we extend the individual representation by adding the rate parameter. Hence the extended individuals are pairs $(x,r)$ consisting of a search point $x \in \{0,1\}^n$ and the rate parameter $r$, which shall indicate that $r/n$ is the mutation rate this individual was created with. 

The extended mutation operator first changes the rate to either $r/F$ or $Fr$ with equal probability ($F>1$). It then performs standard bit mutation with the new rate. 

In the selection step, we choose from the offspring population an individual with best fitness. If there are several such individuals, we prefer individuals having the smaller rate $r/F$, breaking still existing ties randomly. In this winning individual, we replace the rate by $F$ if it was smaller than $F$ to ensure that in the next iterations, the lower of the two rates is at least $1$. We replace the rate by $\rmax = F^{\lfloor \log_F (n/(2F)) \rfloor}$, that is, the largest power of $F$ not exceeding $n/(2F)$, if it was larger than this number. This ensures that in the next iteration, the larger of the two rates is not larger than $n/2$ and that the rate remains a power of $F$ despite the cap.

We formulate the algorithm to start with any initial mutation rate $r^{\init}$ such that $F \le r^{\init} \le n/(2F)$ and $r^{\init}$ is a power of $F$. For the result we shall show in this work, the initial rate is not important, but without this prior knowledge we would strongly recommend to start with the smallest possible rate $r^{\init} = F$. Due to the multiplicative rate adaptation, the rate can quickly grow if this is profitable. On the other hand, a too large initial rate might lead to an erratic initial behavior of the algorithm. 

For the adaptation parameter, we shall use $F=32$ in our runtime analysis. Having such a large adaptation parameter eases the already technical analysis, because now the two competing rates $r/F$ and $Fr$ are different enough to lead to a significantly different performance. For a practical application, we suspect that a smaller value of $F$ is preferable as it leads to a more stable optimization process. The choice of the offspring population size depends mostly on the degree of parallelism one wants to obtain. Clearly, $\lambda$ should be at least logarithmic in $n$ to prevent a too quick loss of the current-best solution. For our theoretical analysis, we require $\lambda \ge C \ln n$ for a sufficiently large constant $C$. 

\begin{algorithm}
\caption{The \oclea with self-adapting mutation rate, adaptation parameter $F>1$, and initial mutation rate $r^{\init}/n$ such that $r^{\init} \in [F,n/(2F)]$ and $r^{\init} = F^i$ for some $i \in \N$.}
\label{alg:onelambda}
	\begin{algorithmic}
		\State Select $x_0$ uniformly at random from $\{0, 1\}^n$.
		\State Set $r_0 \gets r^{\init}$.
		\For{$t \gets 1, 2, \dots$}
			\For{$i \gets 1, \dots, \lambda$}
			% \If{$i \leq \lambda/2$}
			  \State Choose $r_{t,i} \in \{r_{t-1} / F, F r_{t-1}\}$ uniformly at random.
        \State Create $x_{t,i}$ by flipping each bit in $x$ independently with probability $r_{t,i}/n$.
      % \Else
        % \State Create $x_i$ by flipping each bit in a copy of $x$
        % \State independently with probability $2r_t/n$.
      % \EndIf
			\EndFor
			\State Choose $i \in [1..\lambda]$ such that $f(x_{t,i}) = \min_{j \in [1..\lambda]} f(x_{t,j})$; in case of a tie, prefer an $i$ with $r_{t,i} = r_{t-1}/F$; break remaining ties randomly.
			\State $(x_t,r_t) \gets (x_{t,i},r_{t,i})$.
      \State Replace $r_t$ with $\min\{\max\{F, r_t\}, F^{\lfloor \log_F(n/(2F)) \rfloor}\}$.
		\EndFor
	\end{algorithmic}
\end{algorithm}

%\subsection{Runtime Analysis}

The main result of this work is a mathematical runtime analysis of the performance of the algorithm proposed above on the classic benchmark function $\onemax : \{0,1\}^n \to \R$ defined by $\onemax(x) = \sum_{i=1}^n x_i$ for all $x = (x_1, \dots, x_n) \in \{0,1\}^n$. Since such runtime analyses are by now a well-established way of understanding the performance of evolutionary algorithms, we only brief{}ly give the most important details and refer the reader to the textbook~\cite{Jansen13}.

The aim of runtime analysis is predicting how long an evolutionary algorithm takes to find the optimum or a solution of sufficient quality. As implementation-independent performance measure usually the number of fitness evaluations performed in a run of the algorithm is taken. More precisely, the \emph{optimization time} of an algorithm on some problem is the number of fitness evaluations performed until for the first time an optimal solution is evaluated. Obviously, for a \oclea, the optimization time is essentially $\lambda$ times the number of iterations performed until an optimum is generated.

As in classic algorithms analysis, our main goal is an asymptotic understanding of how the optimization time depends on the problems size $n$. Hence all asymptotic notation in the paper will be with respect to $n$ tending to infinity.

\section{Technical Tools}\label{sec:tools}

In this section, we listed several tools which are used in our work. Most of them are standard tools in the runtime analysis of evolutionary algorithms, however, we also prove two new results on occupation probabilities at the end of this section.

\subsection{Elementary Estimates}

We shall frequently use the following estimates.
\begin{lemma}\label{lestimate}
\begin{enumerate}
\item\label{itestimate1} For all $x \in \R$, $1+x \le e^x$.
%\item\label{itestimate2} For all $x \in [0,\frac 23]$, $e^{-5x/3} \le e^{-x-x^2} \le 1-x$. Moreover, for all $x \in [0,\frac 12]$, $e^{-3x/2} \le 1-x$.
\item\label{itestimate2} For all $x \in [0,\frac 23]$, $e^{-x-x^2} \le 1-x$. Moreover, for all $x \in [0,\frac 12]$, $e^{-3x/2} \le 1-x$.
\item\label{itestimate3} Weierstrass product inequality: For all $p_1, \dots, p_n \in [0,1]$, \[1 - \sum_{i=1}^n p_i \le \prod_{i=1}^n (1-p_i).\]
\end{enumerate}
\end{lemma}
All these estimates can be proven via elementary means. We note that the second estimate was proven in~\cite[Lemma 8(c) of the arxiv version]{DoerrGWY17}. The third is usually proven via induction, a possibly more elegant proof via the union bound was given in~\cite{Doerr18bookchapter}.

\subsection{Probabilistic Tools}

In our analysis, we use several standard probabilistic tools including Chernoff bounds. All these can be found in many textbook or the book chapter~\cite{Doerr18bookchapter}. We mention the following variance-based Chernoff bound due to Bernstein~\cite{Bernstein24}, which is less common in this field (but can be found as well in~\cite{Doerr18bookchapter}).

\begin{theorem}\label{tbernstein}
  Let $X_1, \ldots, X_n$ be independent random variables. Let $b$ be such that $E(X_i) - b \le X_i \le E(X_i)+b$ for all $i = 1, \ldots, n$. Let $X = \sum_{i = 1}^n X_i$. Let $\sigma^2 = \sum_{i=1}^n \Var(X_i) = \Var(X)$. Then for all $\lambda \ge 0$,
\begin{align*}
	\Pr(X \ge E(X) + \lambda) &\le 	\exp\bigg(-\frac{\lambda^2}{2(\sigma^2+\frac 13 b \lambda)}\bigg),\\
	\Pr(X \le E(X) - \lambda) &\le 	\exp\bigg(-\frac{\lambda^2}{2(\sigma^2+\frac 13 b \lambda)}\bigg).
\end{align*}
\end{theorem}

We shall follow the common approach of estimating the expected progress and translating this via so-called drift theorems into an estimate for the expected optimization time. We use the variable drift theorem independently found in~\cite{Johannsen10,MitavskiyVariable} in slightly generalized form.

\begin{theorem}[Variable Drift, Upper Bound]
\label{theo:variable-upper}
Given a stochastic process, let $(X_t)_{t\ge 0}$ be a sequence of random variables
obtained from mapping the random state at time~$t$ to a finite set 
$S\subseteq \{0\}\cup [\xmin,\xmax]$, where $\xmin>0$.
Let $T$ be the random variable that denotes the earliest point in time $t \geq 0$
such that $X_t = 0$. 
If there exists a monotone increasing function $h(x)\colon [\xmin,\xmax]\to\R^+$ 
 such that  for all $x\in S$ with 
$\Prob(X_t=x)>0$  we have 
\[
E(X_t-X_{t+1} \mid X_t=x) \ge h(x)
\]
 then for all $x'\in S$ with $\Prob(X_0=x')>0$
\begin{equation*}
  E(T\mid X_0=x') \le
\frac{\xmin}{h(\xmin)} + \int_{\xmin}^{x'} \frac{1}{h(x)} \,\mathrm{d}x.
\end{equation*}
\end{theorem}

Finally, we mention an elementary fact which we shall use as well. See~\cite[Lemma~1]{DoerrD18} for a proof.
\begin{lemma}\label{lem:condbinomial}
  Let $X \sim \Bin(n,p)$ and $k \in [0..n]$. Then $E(X \mid X \ge k) \le E(X) + k$.
\end{lemma}

\subsection{Occupation Probabilities}

To analyze the combined process of fitness and rate in the parent individual, we need a tool that translates a local statement, that is, how the process changes from one time step to the next, into a global statement on the occupation probabilities of the process. Since in our application the local process has a strong drift to the target, Theorem~7 from~\cite{KotzingLissWittFOGA15} is too weak. Also, we cannot assume that the process in each step moves at most some constant distance. For that reason, we need the following stronger statement.
\begin{theorem}[Theorem~2.3 in \cite{Hajek1982}]\label{Hajek1982}
	Suppose that $(\mathcal{F}_k)_{k\ge 0}$ is an increasing family of sub-$\sigma$-fields of $\mathcal{F}$ and $(Y_k)_{k\ge 0}$ is adapted to $(\mathcal{F}_k)$. If 
	\[
	E\bigg(e^{\eta(Y_{k+1}-Y_k)};Y_k>a\mid \mathcal{F}_k\bigg)\le \rho \text{ and } E\bigg(e^{\eta(Y_{k+1}-1)};Y_k\le a\mid \mathcal{F}_k\bigg)\le D,
	\]
	then
	\[
	\Pr\bigg(Y_k\ge b\mid\mathcal{F}_0\bigg)\le \rho^ke^{\eta(Y_0-b)}+\frac{1-\rho^k}{1-\rho}De^{\eta(a-b)}.
	\]
\end{theorem}

We apply this theorem in the following lemma that fit into the case in this paper.
\begin{lemma}\label{lem:occupation}
	Consider a stochastic process~$X_t$, $t\ge 0$, on $\R$ such that for some $p\le 1/25$ 
	the transition probabilities for all $t\ge 0$ satisfy $\Pr(X_{t+1}\ge X_t+a\mid X_t>1)\le p^{a+1}$ for all $a\ge -1/2$ as well as $\Pr(X_{t+1}\ge a+1\mid X_t\le 1)\le p^{a+1}$ for all $a\ge 0$. If $X_0\le 1$ then for all $t\ge 1$ and $k>1$ it holds that 
	\[
	\Prob(X_t\ge 1+k) \le 11 \left(ep\right)^{k}.
	\]
\end{lemma}

\begin{proof}
	We aim at applying Theorem~\ref{Hajek1982}. 
	%To this end, we estimate the moment-generating function of the one-step change $X_{t+1}-X_t$. 
	There are two cases depending on $X_t$: for $X_t\le 1$, using the monotonicity of $e^{\lambda(X_{t+1}-1)}$ with respect to $X_{t+1}-1$, we obtain
	\begin{align*}
	D(p,\lambda)&\coloneqq \E(e^{\lambda (X_{t+1}-1)}\mid X_t\le 1) 
	\le \E(e^{\lambda\max\{\lceil X_{t+1}-1\rceil,0\}}\mid X_t\le 1) \\
	&= e^0\Pr(X_{t+1}\le 1\mid X_t\le 1)+\sum_{a=1}^{\infty}e^{\lambda a}\Pr(a<X_{t+1}\le a+1\mid X_t\le 1)\\
	&\le e^0+\sum_{a=1}^{\infty}e^{\lambda a}\Pr(X_t> a\mid X_t\le 1),
	%\\&\le e^0(1-p) + \sum_{a=0}^{\infty} e^{\lambda(a+1)}p^{a+1}
	%\le 1+\frac{e^{\lambda}p}{1-e^{\lambda}p},
	\end{align*}
	using the assumption that $\Pr(X_{t+1}\ge a+1\mid X_t\le 1)\le p^{a+1}$ for all $a\ge 0$ then
	\begin{align*}
		D(p,\lambda)\le 1+\sum_{a=1}^{\infty} e^{\lambda a}p^{a} =1+\frac{e^{\lambda}p}{1-e^{\lambda}p};
	\end{align*}
%	\begin{align*}
%		&D(p,\lambda)\coloneqq \E(e^{\lambda (X_{t+1}-1)}\mid X_t\le 1) 
%		\le \E(e^{\lambda\max\{X_{t+1}-1,0\}}\mid X_t\le 1) 
%		\le \E(e^{\lambda\max\{\lceil X_{t+1}-1\rceil,0\}}\mid X_t\le 1)\\
%		&\le e^0\Pr(X_{t+1}\le 1\mid X_t\le 1)+\sum_{a=1}^{\infty}e^{\lambda a}\Pr(a<X_{t+1}\le a+1\mid X_t\le 1)\\
%		&\le e^0+\sum_{a=1}^{\infty}e^{\lambda a}\Pr(X_t> a\mid X_t\le 1) \le 1+\sum_{a=1}^{\infty} e^{\lambda a}p^{a} =1+\frac{e^{\lambda}p}{1-e^{\lambda}p},
%		%\\&\le e^0(1-p) + \sum_{a=0}^{\infty} e^{\lambda(a+1)}p^{a+1}
%		%\le 1+\frac{e^{\lambda}p}{1-e^{\lambda}p},
%	\end{align*}
	and for $X_t>1$, using the monotonicity of $e^{\lambda(X_{t+1}-X_t)}$ respect to $X_{t+1}-X_t$, we have
%	\begin{align*}
%		&\rho(p,\lambda)\coloneqq \E(e^{\lambda (X_{t+1}-X_t)}\mid X_t>1) \le \E(e^{\lambda\max\{ X_{t+1}-X_t,-1/2\}}\mid X_t>1)
%		\\&\le e^{-\lambda/2}\Pr\left(X_{t+1}-X_t\le-\frac{1}{2}\mid X_t>1\right)+\sum_{a=0}^{\infty}
%		e^{\lambda a/2}\Pr\left(\frac{a-1}{2}<X_{t+1}-X_t\le\frac{a}{2}\mid X_t>1\right)
%		\\&\le e^{-\lambda/2}+\sum_{a=0}^{\infty}
%		e^{\lambda a/2}\Pr\left(X_{t+1}-X_t>\frac{a-1}{2} \mid X_t>1\right)\\
%		&\le e^{-\lambda/2} +\sum_{a=0}^{\infty} e^{\lambda a/2}p^{(a+1)/2}
%		=\frac{p^{1/2}}{(e^{\lambda}p)^{1/2}}+\frac{p^{1/2}}{1-(e^{\lambda}p)^{1/2}}.	
%	\end{align*}
	\begin{align*}
	\rho(p,\lambda)&\coloneqq \E(e^{\lambda (X_{t+1}-X_t)}\mid X_t>1) \le \E(e^{\lambda\max\{ \lceil 2(X_{t+1}-X_t)\rceil/2,-1/2\}}\mid X_t>1)
	\\&= e^{-\lambda/2}\Pr\left(X_{t+1}-X_t\le-\frac{1}{2}\mid X_t>1\right) \\
	&\quad +\sum_{a=0}^{\infty}
	e^{\lambda a/2}\Pr\left(\frac{a-1}{2}<X_{t+1}-X_t\le\frac{a}{2}\mid X_t>1\right)
	\\&\le e^{-\lambda/2}+\sum_{a=0}^{\infty}
	e^{\lambda a/2}\Pr\left(X_{t+1}-X_t>\frac{a-1}{2} \mid X_t>1\right),
	\end{align*}
	using the assumption that  $\Pr(X_{t+1}\ge X_t+a\mid X_t>1)\le p^{a+1}$ for all $a\ge -1/2$ then
	\begin{align*}
	&\rho(p,\lambda)\coloneqq e^{-\lambda/2} +\sum_{a=0}^{\infty} e^{\lambda a/2}p^{(a+1)/2}
	=\frac{p^{1/2}}{(e^{\lambda}p)^{1/2}}+\frac{p^{1/2}}{1-(e^{\lambda}p)^{1/2}}.
	\end{align*}
	Using $\lambda\coloneqq \ln(1/(ep))$ such that $e^{\lambda}p=1/e$, we have 
	\[
	\rho\coloneqq \rho(p,\lambda) \le e^{1/2}p^{1/2}+\frac{p^{1/2}}{1-e^{-1/2}}\le \frac{e^{1/2}}{5}+\frac{1/5}{1-e^{-1/2}}<0.84,
	\]
	\[
	D \coloneqq D(p,\lambda) \le 1+(1/e)/(1-1/e) <1.6.
	\]
	Theorem 2.3, inequality (2.8) in \cite{Hajek1982} yields with $a\coloneqq 1$ and $b\coloneqq 1+k$ that 
	\begin{align*}
		\Pr(X_t\ge 1+k\mid X_0) &\le \rho^t e^{-\lambda (1+k-X_0)} + \frac{1}{1-\rho} D e^{-\lambda k} 
		\\&\le (ep)^{k} + \frac{1.6}{1-0.84} (ep)^{k}=11(ep)^{k}.
	\end{align*}
\qed\end{proof}

For the simpler case of a random process that runs on the positive integers and that has a strong drift to the left, we have the following estimate for the occupation probabilities. 

\begin{lemma}\label{lem:occsimple}
  Consider a random process defined on the positive integers $1, 2, \dots$. Assume that from each state $i$ different from $1$, only the two neighboring states $i-1$ and $i+1$ can be reached (and there is no self-loop on state $i$). From state $1$, only state $2$ can be reached and the process can stay on state $1$. Let $p_i$ be an upper bound for the transition probability from state $i$ to state $i+1$ (valid in each iteration regardless of the past). Assume that 
  \[p_{i-1} \ge \frac{p_i}{1-p_i}\] 
  holds for all $i \ge 2$. Assume that the process starts in state $1$. Then at all times, the probability to be in state $i$ is at most \[q_i := \prod_{j=1}^{i-1} \frac{p_j}{1-p_j},\] where as usual we read the empty product as $q_1 = 1$.
\end{lemma}

\begin{proof}
  The claimed bound on the occupation probabilities is clearly true at the start of the process. Assume that it is true at some time. By this assumption and the assumptions on the process, the probability to be in state $i \ge 2$ after one step is at most
  \begin{align*}
  q_{i-1} p_i + q_{i+1} &= q_{i-1}\left(p_i + \frac{p_{i-1}}{1 - p_{i-1}}\frac{p_{i}}{1 - p_{i}}\right) \\
  &\le q_{i-1}\left(\frac{p_i}{1-p_i} + \frac{p_{i-1}}{1 - p_{i-1}}\frac{p_{i}}{1 - p_{i}}\right) \\
  &= q_{i-1}\left(\frac{p_{i}(1-p_{i-1})}{(1-p_{i-1})(1-p_i)} + \frac{p_{i-1}}{1 - p_{i-1}}\frac{p_{i}}{1 - p_{i}}\right) \\
  &\le q_{i-1} \frac{p_{i}}{(1-p_{i-1})(1-p_i)} \le q_{i-1} \frac{p_{i-1}}{1 - p_{i-1}} = q_{i}.
  \end{align*}
  Trivially, the probability to be in state $1$ after one step is at most $q_1 = 1$.
  Hence, by induction over time, we see that $q_i$ is an upper bound for the probability to be in state $i$ at all times.   
\qed\end{proof}

\section{Main Result and Proof}
\label{sec:main}

We can now state precisely our main result and prove it.

\begin{theorem}
\label{theo:main}
Let $\lambda \ge  C \ln n$ for a sufficiently large constant $C >0$ and %$\lambda = \exp(o(\sqrt n))$
$\lambda = n^{O(1)}$. 
Let $F=32$. Then the expected number 
of generations the self-adapting \oclea takes to optimize \om is 
\[O\left(\frac{n}{\log \lambda} + \frac{n \log n}{\lambda}\right).\]
This corresponds to an expected number of fitness evaluations of  $O(n\lambda / \log \lambda + n \log n)$.
\end{theorem}

The proof of this theorem is based on a careful, technically demanding drift analysis of both 
the current \om-value~$k_t$ (which is also the 
fitness distance, recall that our goal is the minimization of the objective function) and the current rate~$r_t$ of the parent. In very rough terms, a similar division  of the run as in 
\cite{DoerrGWY18} into regions of large \om-value, the 
far region (Section~\ref{sec:far}), and of small \om-value, the near region (Section~\ref{sec:near}) is made. The middle 
region considered in \cite{DoerrGWY18} is subsumed under the far region here.

In the remainder of our analysis, we assume that $n$ is sufficiently large, that $\lambda \ge  C \ln n$ with a sufficiently large constant $C$, and that $\lambda=n^{O(1)}$.

\subsection{The Far Region}
\label{sec:far}
In this section, we analyze the optimization behavior of our self-adaptive \oclea in the regime where the fitness distance $k$ is at least $n/\lambda$. Due to our assumption $\lambda \ge C \ln n$, it is very likely to have at least one
 copy of the parent among $\lambda$ offspring when $r=O(\ln\lambda)$. Thus the \oclea works almost the same as the $(1+\lambda)$~
EA when $r$ is small, but can lose fitness in general. The following lemma is crucial in order to analyze the drift of the rate depending 
on~$k$, which follows a similar scheme as with the $(1+\lambda)$~EA proposed in~\cite{DoerrGWY18}. 

Roughly speaking, the rate leading to optimal fitness progress is $n$ for $k \ge n/2 + \omega(\sqrt{n \ln(\lambda)})$, 
$n/2$ for $k = n/2 \pm o(\sqrt{n \log(\lambda)})$, and then the optimal rate quickly drops to $r = \Theta(\log \lambda)$ 
when $k \le n/2-\eps n$.
%  We take the multiplicative factor on rate $F=32$ so that it will be easier 
%for the smaller rate $r/F$ to copy parent. 

To ease the representation, we first define two fitness dependent bounds $L(k)$ and $R(k)$.
\begin{definition}\label{define-L-U}
	Let %$n$ be sufficiently large, 
	$n/\ln\lambda<k<n/2$ and $F=32$. We define
	$L(k):=(F\ln(en/k))^{-1}$ and $U(k):=n(2n-k)/(22(n-2k)^2)$.
\end{definition}
According to the definition, both $L(k)$ and $R(k)$ monotonically increase when $k$ increases.

\begin{lemma}\label{far-drift-rate}
	Let $F=32$. Consider an iteration of the self-adaptive \oclea with current fitness distance~$k$ and current rate~$r$.  
	
	%and $n$ sufficiently large.
	%Let $\lambda\ge \ln n $ and $\lambda=o(\exp(\sqrt{n}))$.
	%Let $L(k):=(F\ln(en/k))^{-1}$ and $U(k):=n(2n-k)/(20(n-2k)^2)$, $F=32$ and $n$ sufficiently large.
	Then: 
	\begin{enumerate}
		\item \label{inc}If $n/\ln\lambda<k$ and $F\le r \le L(k)\ln\lambda$, the probability that all best offspring have been created with rate $Fr$ is at least $1-O(\ln^3(\lambda)/\lambda^{1/(4\ln\ln\lambda)})$.
		
		\item \label{dec}\label{itfardrifttoolarge}If $k<n/2$ and $n/(2F)\ge r \ge U(k)\ln\lambda$, then the probability that all best offspring have been created with rate $r/F$ is at least $1-\lambda^{1-(23/22)r/(U(k)\ln\lambda)}$.
		
	\end{enumerate}
\end{lemma}
\begin{proofof}{Lemma~\ref{far-drift-rate} part~\ref{inc}}
	Let $q(k,i,r)$ and $Q(k,i,r)$ be the probability that standard bit mutation with mutation rate $p=r/n$ creates from a parent with fitness distance $k$ an offspring with fitness distance exactly $k-i$ and at most $k-i$, respectively. Then
	\begin{equation}\label{expand-q}
		q(k,i,r)=\sum_{j=0}^{k-i}\binom{k}{i+j}\binom{n-k}{j}p^{i+2j}(1-p)^{n-i-2j}
	\end{equation}	
	and $Q(k,i,r)=\sum_{j=i}^{k}q(k,j,r)$. 
	%Since $\ln(1-x)\ge -x-x^2$ for all $0\le x \le 1/2$, we first notice that for $r=o(\sqrt{n})$ we have $(1-r/n)^n\ge \exp(-r-r^2/n)\ge (1-o(1))e^{-r}$. 
	We aim at finding $i$ such that $Q(k,i,Fr)\ge \ln(\lambda)/\lambda$ while $Q(k,i,r/F)=O(\ln^3(\lambda)/\lambda^{1+1/4(\ln\ln\lambda)})$. Then we use these to bound the probability that at least one offspring using rate $Fr$ obtains a progress of $i$ or more while at the same time all offspring using rate $r/F$ obtains less than $i$ progress.
	Let $i^*$ be the largest $i$ such that $Q(k,i,Fr)\ge \ln(\lambda)/\lambda$. Using the fact that $\ln(1-x)\ge-x-x^2$ for all $0<x<2/3$, we notice that $(1-Fp)^{n-i}\ge(1-Fp)^{n}\ge e^{-Fr-(Fr)^2/n}$. By the assumption that $r\le L(k)\ln\lambda\le \ln\lambda$, we obtain $(Fr)^2/n=O(\ln^2\lambda/n)=o(1)$. Thus  $(1-Fp)^{n-i}=(1-o(1))e^{-Fr}$. We also notice that $\binom{k}{i}=(k/i)((k-1)/(i-1))\cdots(k-i+1)>(k/i)^{i-1}(k-i)= (k/i)^{i}((k-i)i/k)>2(k/i)^i$ for $2<i<k-2$. Thus for $i>2$ we can bound $Q(k,i,Fr)$ by
	\begin{align}\label{eqn-Q}
		& Q(k,i,Fr)\ge q(k,i,Fr)\ge\binom{k}{i}(Fp)^{i}(1-Fp)^{n-i}
		\ge \left(\frac{k}{i}\cdot\frac{Fr}{n}\right)^{i}e^{-Fr}.
	\end{align}
	Let $i=\max\{(F-1)r,\ln\lambda/(8\ln\ln\lambda)\}$. We prove $i^*\ge i$ by distinguishing between two cases according to which argument maximizes $i$.
	
	If $i=\ln\lambda/(8\ln\ln\lambda)$, then $r\le i/(F-1)$ and $Fr\le 2i$. Referring to inequality \eqref{eqn-Q} and using the fact that $k/n\ge1/\ln\lambda$, $i<\ln\lambda$, and $\ln\ln(\lambda)>1$, we obtain
	\begin{align*}
		&\ln(Q(k,i,Fr))\ge i\ln\left(\frac{k}{in}\right)-Fr\ge -i\ln(\ln^2\lambda)-2i\\
		&=-2i\ln\ln\lambda-2i>-4i\ln\ln\lambda= -\frac{\ln\lambda}{2}\ge \ln\left(\frac{\ln\lambda}{\lambda}\right)
	\end{align*}
	and thus $Q(k,i,Fr)\ge \ln(\lambda)/\lambda$.

	If $i=(F-1)r$, then $r\ge \ln\lambda/(8(F-1)\ln\ln\lambda)$ since  $F$ is a constant.
	Using $r\le L(k)\ln\lambda$, we obtain $\ln\lambda\ge\ln(en/k)Fr$ which is equivalent to $(k/en)^{Fr}\ge 1/\lambda$. Furthermore,  $(k/n)^i>(k/n)^{Fr}$ since $i=(F-1)r<Fr$. Thus
	\begin{align*}
	& Q(k,i,Fr)
	\ge \left(\frac{k}{i}\cdot\frac{Fr}{n}\right)^{i}e^{-Fr}\ge \left(\frac{F}{F-1}\right)^{(F-1)r}\left(\frac{k}{en}\right)^{Fr}\ge 2^r\left(\frac{k}{en}\right)^{Fr}\ge\frac{\ln\lambda}{\lambda}.
	\end{align*}
	Since $Q(k,i,r)$ is decreasing in $i$, we obtain $i^*\ge\max\{(F-1)r,\ln\lambda/(8\ln\ln\lambda)\}$. Using a Chernoff bound and recalling that the expected number of flipped bits is bounded by $FL(k)\ln\lambda\le\ln\lambda/\ln(2e)$, we notice that $i^*\le \ln\lambda$. This upper bound will be used to estimate $Q(k,i^*,Fr)/Q(k,i^*+1,Fr)$ in the following part of the proof.
	
	We now prove that $Q(k,i^*,r/F)=o(1/\lambda)$. By comparing each component in $q(k,i,r/F)$ and $q(k,i,Fr)$, and applying Lemma~\ref{lestimate}~\ref{itestimate2} to estimate $(1-Fr/n)^n$ and $(1-r/(Fn))^n$ with $r=O(\ln\lambda)=o(n^{1/2})$ for large enough $n$, we obtain 
	\begin{align*}
		\frac{q(k,i,Fr)}{q(k,i,r/F)}&\ge F^{2i} \frac{(1-Fr/n)^n}{(1-r/(Fn))^n}\\
		&\ge \big(1-o(1)\big)F^{2i}e^{-(F-1/F)r}>F^{2i}e^{-Fr}.
	\end{align*}
	Therefore $Q(k,i^*,Fr)/Q(k,i^*,r/F)\ge F^{2i^*}e^{-Fr}= \exp(2i^*\ln F-Fr)\ge \exp(2i^*\ln F-Fi^*/(F-1))>\exp(3i^*)>\lambda^{1/(4\ln\ln\lambda)}$, where in the first inequality, we use the fact that $i^*\ge (F-1)r$. To prove $Q(k,i^*,r/F)=O(\ln^3(\lambda)/\lambda^{1+1/4(\ln\ln\lambda)})$, we first show $Q(k,i^*,Fr)/Q(k,i^*+1,Fr)=O(\ln^2\lambda)$. 
	Then we use this to bound $Q(k,i^*,Fr)=O(\ln^3(\lambda)/\lambda)$  according to the definition of $i^*$. Finally we obtain $Q(k,i^*,r/F)\le Q(k,i^*,Fr)/\lambda^{1/(4\ln\ln\lambda)}=O(\ln^3(\lambda)/\lambda^{1+1/4(\ln\ln\lambda)})$.  It remains to bound $Q(k,i^*,Fr)/Q(k,i^*+1,Fr)$.  We show that the majority of $q(k,i,r)$ are from the first $3r$ terms in the summation of equation \eqref{expand-q}. Let $q(k,i,r)_j$ denote the $j$-th item $\binom{k}{i+j}\binom{n-k}{j}p^{i+2j}(1-p)^{n-i-2j}$ in equation \eqref{expand-q}. Then
	\[
	\frac{q(k,i,r)_{j+1}}{q(k,i,r)_j}=\frac{k-i-j}{i+j+1}\cdot\frac{n-k-j}{j+1}\cdot p^2\cdot(1-p)^{-2}\le\frac{r^2}{(i+j+1)(j+1)}.
	\] 
	If $j>3r$, then $r^2/((i+j+1)(j+1))<1/9$, and thus
	\begin{eqnarray*}
	q(k,i,r)&\le& \left(\sum_{j=0}^{3r}q(k,i,r)_j\right)+q(k,i,r)_{3r}\left(\sum_{j=3r+1}^{k-i}(1/9)^{j-3r}\right)\\
	&\le &\left(\sum_{j=0}^{3r}q(k,i,r)_j\right)+q(k,i,r)_{3r}\cdot\frac{1/9}{1-1/9}\\
	&=&\left(\sum_{j=0}^{3r-1}q(k,i,r)_j\right)+\frac{9}{8} \cdot q(k,i,r)_{3r}\le \frac{9}{8}\sum_{j=0}^{3r}q(k,i,r)_j.
	\end{eqnarray*}
	%$q(k,i,r)\le \sum_{j=0}^{3r}q(k,i,r)_j+q(k,i,r)_{3r}\sum_{j=3r+1}^{k-i}(1/9)^{j-3r}\le \sum_{j=0}^{3r}q(k,i,r)_j+q(k,i,r)_{3r}/8\le (9/8)\sum_{j=0}^{3r}q(k,i,r)_j.$
	We notice that 
	\[
	\frac{q(k,i+1,r)_j}{q(k,i,r)_j}=\frac{\binom{k}{i+j+1}\binom{n-k}{j}p^{i+2j+1}(1-p)^{n-i-2j-1}}{\binom{k}{i+j}\binom{n-k}{j}p^{i+2j}(1-p)^{n-i-2j}}=\frac{(k-i-j)p}{(i+j+1)(1-p)},
	\]
	using the fact that $\sum_{j=0}^{3r}q(k,i,r)_j\le q(k,i,r)\le(9/8)\sum_{j=0}^{3r}q(k,i,r)_j$ for all $(k,i,r)$, 
	we compute
	\[
	\frac{q(k,i^*+1,Fr)}{q(k,i^*,Fr)}\ge \frac{\sum_{j=0}^{3Fr}q(k,i^*+1,Fr)_j}{(9/8)\sum_{j=0}^{3Fr}q(k,i^*,Fr)_j}\ge \frac{8}{9}\cdot\frac{k-i^*-3Fr}{i^*+3Fr+1}\cdot \frac{p}{1-p}.
	\]
	Since $i^*\ge (F-1)r$, $i^*\le \ln\lambda$, and $k\ge n/\ln\lambda=\omega(\ln\lambda)$, we obtain
	\[
	\frac{q(k,i^*+1,Fr)}{q(k,i^*,Fr)}=\Omega\left(\frac{kp}{i^*}\right)=\Omega\left(\frac{kr}{i^*n}\right)=\Omega\left(\frac{1}{\ln^2\lambda}\right).
	\]
	Consequently we have $q(k,i^*,Fr)/Q(k,i^*+1,Fr)\le q(k,i^*,Fr)/q(k,i^*+1,Fr)=O(\ln^2\lambda)$ and
	\[
	\frac{Q(k,i^*,Fr)}{Q(k,i^*+1,Fr)}=1+\frac{q(k,i^*,Fr)}{Q(k,i^*+1,Fr)}= O(\ln^2\lambda).
	\]
	So finally $Q(k,i^*,Fr)=O(\ln^3(\lambda)/\lambda)$ due to the definition of $i^*$, and
	% \vspace{-2mm}
	\[
	Q(k,i^*,r/F)\le \frac{Q(k,i^*,Fr)}{F^{2i^*}e^{-Fr}}=O\left(\frac{\ln^3\lambda}{\lambda\cdot\lambda^{1/(4\ln\ln\lambda)}}\right).
	\]
	A simple union bound shows that with probability	$1-O(\ln^3(\lambda)/\lambda^{1/(4\ln\ln\lambda)})$, no offspring of rate~$r/F$ manages to obtain a progress of $i^*$ or more. However, the probability that an offspring has rate $Fr$ and obtains at least $i^*$ progress is $\ln(\lambda)/(2\lambda)$. Thus the probability that no offspring generated with rate $Fr$ achieves a progress of at least $i^*$ is at most $(1-\ln(\lambda)/(2\lambda))^{\lambda}\le \lambda^{-1/2}=o(\ln^3(\lambda)/\lambda^{1/(4\ln\ln\lambda)})$. This proves the first statement of the lemma.
	%By Chernoff bound, the offspring size for rate $Fr$ is at least $\lambda/3$ with probability $1-\exp(-(\lambda/2)(1/2-1/3)^2/2)=1-\exp(-\lambda/144)$. Further more, with probability at least $1-(1-\ln(\lambda)/\lambda)^{\lambda/3}\ge 1-\lambda^{-1/3}$,	one of the $\lambda/3$ offspring of rate $Fr$ manages to obtain a progress of $i^*$ or more. This proves the first statement of the lemma.
\end{proofof}
\begin{proofof}{Lemma~\ref{far-drift-rate} part~\ref{dec}}
	For $\tilde{r}\in\{r/F,Fr\}$ let the random variable $X(k,\tilde{r})$ denote the number of flipped bits in $k$ ones and $Y(k,\tilde{r})$ denote the number of flipped bits in $n-k$ zeros when applying standard bit mutation with probability $p=\tilde{r}/n$. Let $Z(k,\tilde{r}):=Y(k,\tilde{r})-X(k,\tilde{r})$ denote the improvement in fitness. Let $Z^*(k,\tilde{r})$ denote the minimal $Z(k,\tilde{r})$ among all offspring which apply rate $\tilde{r}$.  $\E(Z(k,\tilde{r}))=(n-k)\tilde{r}/n-k\tilde{r}/n=(n-2k)\tilde{r}/n$. Our aim is to find a $\beta$ such that $\Pr\left(Z(k,r/F)\le \beta\right)=\Theta(1)$ while $\Pr\left(Z(k,Fr)\le \beta\right)=o(1/\lambda)$, and use this to obtain a high value for $\Pr\left(Z^*(k,r/F)< Z^*(k,Fr)\right)$. 
	
	Let $\beta:=\E(Z(k,r/F))$. 
	We notice that $\Pr(X(k,r/F)>\E(X(k,r/F))-1)\ge 1/2$ since the median of binomial distribution $X(k,r/F)$ is $\lfloor E(X(k,r/F) \rfloor$ or $\lceil E(X(k,r/F)\rceil$. Applying Lemma~8 in \cite{Doerr18exceedexp} to $\Pr(Y(k,r/F)<\E(Y(k,r/F))-1)$ with $E(Y(k,r/F))=\Omega(\ln\lambda)=\omega(1)$ by assumption $r\ge U(k)\ln\lambda$ and $E(Y(k,r/F))<(n-k)/2$, we obtain for $n$ sufficiently large that
	\begin{align}
		&\Pr\Big(Y(k,r/F)<\E(Y(k,r/F))-1\Big)\nonumber
		\\&\ge\frac{1}{2}-\sqrt{\frac{n-k}{2\pi\lfloor (n-k)p \rfloor(n-k-\lfloor (n-k)p\rfloor)}}>\frac{2}{5}.\label{2/5}
	\end{align}
	Thus $\Pr(Z(k,r/F)\le\beta)>(1/2)(2/5)=1/5$. We use Bernstein's inequality (version Lemma~\ref{lemma:prob-single-offspring}) to bound $\Pr\left(Z(k,Fr)\le \beta\right)$ and obtain
	\[
	\Pr\Big(Z(k,Fr)\le E(Z(k,Fr))-\Delta\Big)\le\exp\left(-\frac{\Delta^2}{2(\Var(Z(k,Fr))+\Delta/3)}\right) \text{ for all }\Delta>0.
	\]
	With $\Delta=E(Z(k,Fr))-\beta=(n-2k)(Fr/n-r/(Fn))=(n-2k)(F^2-1)r/(Fn)$ and $\Var(Z(k,Fr))=Fr(1-Fr/n)<Fr$, we compute
	\begin{eqnarray*}
		Pr\Big(Z(k,Fr)\le\beta \Big)&\le& \exp\left(-\frac{1}{2}\cdot \frac{(F^2-1)^2(n-2k)^2r^2}{F^2n^2(Fr+(n-2k)(F^2-1)r/(3Fn))}\right)\\
		&=&\exp\left(-\frac{1}{2}\cdot \frac{(F^2-1)^2(n-2k)^2r}{F^3n^2+Fn(n-2k)(F^2-1)/3}\right)\\
		&\le&\exp\left(-\frac{3}{2}\cdot \frac{(F^2-1)^2(n-2k)^2r}{3F^3n^2+F^3n(n-2k)}\right)\\
		&=&\exp\left(-\frac{3}{4}\cdot \frac{(F^2-1)^2(n-2k)^2r}{F^3n(2n-k)}\right).
	\end{eqnarray*}
	Given $F=32$ and $r\ge U(k)\ln\lambda$ then 
	\begin{align*}
		&\Pr\Big(Z(k,Fr)\le \beta\Big)
		<\exp\left(-\frac{23.9(n-2k)^2r}{n(2n-k)}\right)<\lambda^{-\frac{23.9r}{22U(k)\ln\lambda}}.
	\end{align*}
	With a simple union bound, we obtain $\Pr(Z^*(k,Fr)\le \beta)<\lambda\Pr(Z(k,Fr)\le \beta)<\lambda^{1-23.9r/(22U(k)\ln\lambda)}$. The probability that an offspring has rate $r/F$ and obtains $\beta$ is at least $(1/2)(1/5)=1/10$. Thus the probability that no offspring generated with $r/F$ has a $Z$-value of at least $\beta$ is at most $(1-1/10)^{\lambda}=\exp(-\Theta(\lambda))$. Therefore $\Pr(Z^*(k,Fr)< Z^*(k,r/F))<\lambda^{1-23.9r/(22U(k)\ln\lambda)}(1-\exp(-\Theta(\lambda)))=o(\lambda^{1-23r/(22U(k)\ln\lambda)})$, which means with probability at least $1-\lambda^{1-(23r/(22U(k)\ln\lambda)}$ all best offspring have been created with rate $r/F$.
\end{proofof}

Lemma~\ref{far-drift-rate} will be crucial in order to bound the expected progress on fitness in the far region. 
We notice that $\ln\lambda=o(\sqrt{n})$ in the lemma we may allow $r>\ln\lambda$ when $k$ is large and $r=\Theta(n)$ when 
$k=n/2-\Theta(\sqrt{n\ln\lambda})$. It is easy to show a positive progress on fitness for $r<\ln\lambda$ 
since there will be sufficiently many offspring that do not flip zeroes. 
When $r\ge \ln\lambda$ we expect all offspring to flip zeros, but we can still show a positive drift when $k>7n/20$, as 
stated in the following lemma. The idea is that 
%This is because if we look at the flipping results of $k$ zeros and $k$ ones, the expected deviation towards more zeros and
the standard variation of the number of flipping ones is $\sqrt{kr/n(1-r/n)}=\Theta(\sqrt{r})$. 
This makes a deviation  compensating bad flips among the remaining $n-2k$ zeros likely enough.

\begin{lemma}\label{far-prob}
	Let %$n$ be large enough, 
	$7n/20 \le k<n/2$, $F=32$ and $\alpha=10^{-4}$. 
	%Let $\lambda\ge \ln n $ and $\lambda=o(\exp(\sqrt{n}))$. 
	Assume $r\le \min\{n^2\ln\lambda/(12(n-2k)^2),n/(2F)\}$. Assume that from a parent with fitness distance $k$ we generate an offspring using standard bit mutation with mutation rate $p=r/n$. Then the probability that this offspring has a fitness distance of at most $k-s$ with $s:=\alpha(\min\{\ln\lambda,r\}+(n-2k)r/n)$, is at least $\lambda^{-0.98}$.
\end{lemma}
\begin{proof}
	We first look at the case when $r<1/(2\alpha)$. In this case $s\le \alpha(r+(n-2k)r/n)\le \alpha(2r)<1$. Then the probability that this offspring has a fitness distance of $k-1>k-s$ is at least
	\begin{align*}
		\binom{k}{1}\left(\frac{r}{n}\right)^1\left(1-\frac{r}{n}\right)^{n-1}=\Theta(e^{-r})=\omega(\lambda^{-0.98}).
	\end{align*}
	Therefore it remains to consider $r\ge 1/(2\alpha)$.
	
	Let random variables $X$ and $Y$ denote the number of flips in $k$ one-bits and $(n-k)$ zero-bits, respectively, in an offspring using rate $p=r/n$. Then $X-Y$ is the decrease of fitness distance. $X$ and $Y$ follow binomial distributions $\Bin(k,p)$ and $\Bin(n-k,p)$, respectively.
	Let
	\begin{align*}
		&B(x)\coloneqq\Pr(X=x)=\binom{k}{x}p^x(1-p)^{k-x} \text{ for all } x\in\{0,1,\dots,k\},\\
		&F(x)\coloneqq\Pr(X\ge x)=\sum_{i=\lceil x\rceil}^{k}B(i) \text{ for all } x\in[0,k].
	\end{align*}
	Since $r\ge 1/(2\alpha)\ge 5000$ and $n\le 20k/7$, then $p=r/n\ge 5000\cdot7/(20k)=1750/k$. Using this and the fact that $p\le 1/(2F)$,
	we apply Lemma 8 in \cite{Doerr18exceedexp} and obtain
	 \[
	 \Pr(X>E(X))>\frac{1}{2}-\sqrt{\frac{k}{2\pi\lfloor kp \rfloor(k-\lfloor kp\rfloor)}}>\frac{1}{2}-\sqrt{\frac{1}{2kp}}>\frac{2}{5}
	 \]
	 Similarly $\Pr(Y\le E(Y))=2/5$. Since $E(X-Y)=kp-(n-k)p=-(n-2k)p$, we  bound
	\begin{align*}
		\Pr\bigg(X-Y\ge s\bigg)&\ge \Pr\bigg(X\ge E(X)+(n-2k)p+s\bigg)\Pr\bigg(Y\le E(Y)\bigg)\\
		&\ge \frac{2}{5}F\bigg(kp+(n-2k)p+s\bigg).
	\end{align*}
	Let $\delta\coloneqq\lceil (n-2k)p+s\rceil$, $u\coloneqq  kp $ and $\tilde{u}:=\lceil u \rceil$. We notice that $u=rk/n\ge (1/(2\alpha))(7/20)=1750$. Furthermore, we have $\delta<\tilde{u}-2<u$ since 
	\begin{eqnarray*}
	\delta&= &\lceil(n-2k)p+s\rceil<(1+\alpha)(n-2k)p+\alpha\min\{\ln\lambda,r\}+1\\
	&\le& (1+\alpha)\frac{n-2k}{n}r+\alpha r+1
	\le\left(\left(1+\alpha\right)\frac{3}{10}+\alpha\right)r+1\\
	&=&\frac{3+13\alpha}{10}\cdot\frac{n}{k}\cdot u+1
	<\frac{3+13\alpha}{10}(3u)+1<0.91u+1\\
	&=& u-0.09u+1\le u-(0.09\cdot 1750-1 )=u-156.5.
 	\end{eqnarray*}
%	\begin{align*}
%	&(n-2k)p+s=(1+\alpha)(n-2k)p+\alpha\min\{\ln\lambda,r\}\\
%	&\le (1+\alpha)\frac{n-2k}{n}r+\alpha r\le\left(\left(1+\alpha\right)\frac{3}{10}+\alpha\right)r\\
%	&=\frac{3+13\alpha}{10}\cdot\frac{n}{k}\cdot u
%	<\frac{3+13\alpha}{10}(3u)<0.91u.
% 	\end{align*}
	We aim at proving $F(u+\delta)=\omega(\lambda^{-0.98})$ to obtain this lemma. If $F(u+\delta)=\Theta(1)$ then the conclusion holds. It remains to consider $F(u+\delta)=o(1)$ while $F(u)-F(u+\delta)\ge 2/5-o(1)$ as stated in equation \eqref{2/5}. 
	For any $x\in\Z_{\ge u}$ we have
	\[
	\frac{B(x+1)}{B(x)}=\frac{k-x}{x+1}\cdot \frac{p}{1-p}\le \frac{u-up}{u-up+1-p}<1.
	\]
	Since $\tilde{u}=\lceil u \rceil$ then $B(\tilde{u})>B(\tilde{u}+1)>\cdots>B(k)$, and thus $F(u+\delta)\ge  \delta  B(\tilde{u}+2\delta )$ as well as
	$F(u)-F(u+\delta)\le  \delta  B(\tilde{u})$. Using the fact that  $p/(1-p)=u/(k-u)$ and $\tilde{u}-1<u$, we see that
	\begin{align*}
	\frac{B(\tilde{u}+2 \delta)}{B(\tilde{u})}&=\frac{(k-\tilde{u})\cdots(k- (\tilde{u}+2 \delta)+1)}{(\tilde{u}+1)\cdots( \tilde{u}+2\delta)}\cdot\frac{p^{2\delta}}{(1-p)^{2\delta}}\\
	&\ge\frac{(k- (\tilde{u}-1)-2 \delta)^{2\delta}}{(\tilde{u}+1)\cdots( \tilde{u}+2\delta)}\cdot\frac{u^{2\delta}}{(k-u)^{2\delta}}
	\ge\left(1-\frac{2 \delta}{k-u}\right)^{2\delta}\frac{u^{2\delta}}{(\tilde{u}+1)\cdots(\tilde{u}+2\delta)}.
	%\\&\ge\frac{(k- (u+2 \delta))^{2\delta}}{(\tilde{u}+1)\cdots( \tilde{u}+2\delta)}\cdot\frac{u^{2\delta}}{(k-u)^{2\delta}}=\left(1-\frac{2 \delta}{k-u}\right)^{2\delta}\frac{u^{2\delta}\tilde{u}!}{(\tilde{u}+2\delta)!}
	\end{align*}
	We compute the following factorials using Robbins's Stirling's approximation in \cite{robbins55}
	\begin{align*}	
	& (\tilde{u}+2\delta)!\le \sqrt{2\pi(\tilde{u}+2\delta)}\left(\frac{\tilde{u}+2\delta}{e}\right)^{\tilde{u}+2\delta}\exp\left(\frac{1}{12(\tilde{u}+2\delta)}\right),\\
	& \tilde{u}!\ge \sqrt{2\pi \tilde{u}}\left(\frac{ \tilde{u}}{e}\right)^{ \tilde{u}}\exp\left(\frac{1}{12 \tilde{u}+1}\right).
	\end{align*}
	Notice that $12 \tilde{u}+1<12( \tilde{u}+2\delta)$, we obtain
	\[
	\frac{1}{ (\tilde{u}+1)\cdots(\tilde{u}+2\delta)}=\frac{\tilde{u}!}{(\tilde{u}+2\delta)!}\ge\sqrt{\frac{\tilde{u}}{\tilde{u}+2\delta}}\frac{\tilde{u}^{\tilde{u}}e^{2\delta}}{(\tilde{u}+2\delta)^{\tilde{u}+2\delta}}\ge \sqrt{\frac{\tilde{u}}{\tilde{u}+2\delta}}\frac{u^{\tilde{u}}e^{2\delta}}{(\tilde{u}+2\delta)^{\tilde{u}+2\delta}}.
	\]
	Therefore
	\begin{align*}
	&\frac{B(\tilde{u}+2\delta)}{B(\tilde{u})}\ge \left(1-\frac{2\delta}{k-u}\right)^{2\delta}\sqrt{\frac{\tilde{u}}{\tilde{u}+2\delta}}\frac{u^{\tilde{u}+2\delta}e^{2\delta}}{(\tilde{u}+2\delta)^{\tilde{u}+2\delta}}\\
	&=\sqrt{\frac{\tilde{u}}{\tilde{u}+2\delta}}\exp\left(2\delta\ln\left(1-\frac{2\delta}{k-u}\right)+(\tilde{u}+2\delta)\ln\left(\frac{u}{\tilde{u}+2\delta}\right)+2\delta\right)\\
	&\ge\sqrt{\frac{\tilde{u}}{\tilde{u}+2\delta}}\exp\left(2\delta\ln\left(1-\frac{2\delta}{k-u}\right)+(\tilde{u}+2\delta)\ln\left(1-\frac{2\delta+1}{\tilde{u}+2\delta}\right)+2\delta\right).
	\end{align*}
	We notice that $2\delta/(k-u)\le 2\delta/(2Fu-u)= 2\delta/(63u)<2/63<1/2$ and $(2\delta+1)/(\tilde{u}+2\delta)<(2\delta+1)/(3\delta+2)<2/3$. 
	Referring to Lemma~\ref{lestimate}, we compute
	\begin{align}
	2\delta\ln\left(1-\frac{2\delta}{k-u}\right)&\ge -\frac{3}{2}\cdot\frac{4\delta^2}{k-u}=-\frac{6\delta^2}{u/p-u}\ge-\frac{6\delta^2}{2Fu-u}\ge-\frac{\delta^2}{10u},\nonumber\\
	(\tilde{u}+2\delta)\ln\left(1-\frac{2\delta+1}{\tilde{u}+2\delta}\right)&\ge-(2\delta+1)-\frac{(2\delta+1)^2}{\tilde{u}+2\delta}\ge-2\delta-\frac{4\delta^2}{u}-3,\nonumber\\
	%\frac{B(\tilde{u}+2\delta)}{B(\tilde{u})}&\ge\sqrt{\frac{1}{3}}\exp\left(-\frac{41\delta^2}{10u}-3\right).
	\frac{B(\tilde{u}+2\delta)}{B(\tilde{u})}&\ge\sqrt{\frac{\tilde{u}}{\tilde{u}+2\delta}}\exp\left(-\frac{41\delta^2}{10u}-3\right)\ge \sqrt{\frac{1}{3}}e^{-3}\exp\left(-\frac{41\delta^2}{10u}\right).\label{ratio_B}
	\end{align}
	where the last inequality used that $\tilde{u}/(\tilde{u}+2\delta)\ge1/3 $ since $\delta\le \tilde{u}$.
	Using $n/k\le 20/7$ and $r\le n^2\ln\lambda/(12(n-2k)^2)$, we obtain
	\begin{align*}
		\frac{41\delta^2}{10u}&=\frac{41n}{10k}\cdot \frac{\lceil(1+\alpha)((n-2k)/n) r+\alpha\min\{\ln\lambda,r\}\rceil^2}{r}\\
		&\le \frac{41n}{10k}\cdot \frac{\left((1+\alpha)((n-2k)/n) r+\alpha\min\{\ln\lambda,r\}+1\right)^2}{r}\\
		&\le \frac{41n}{10k}\cdot\left(\frac{\left((1+\alpha)((n-2k)/n) r+\alpha\min\{\ln\lambda,r\}\right)^2}{r}+2(1+\alpha)\frac{n-2k}{n}+2\alpha+\frac{1}{r}\right)\\
		&\le \frac{82}{7}\cdot\left((1+\alpha)^2\left(\frac{n-2k}{n}\right)^2r +2(1+\alpha)\frac{n-2k}{n}\alpha\ln\lambda+\alpha^2\ln\lambda+1\right)\\
		&\le \frac{82}{7}\cdot\left(\frac{(1+\alpha)^2\ln\lambda}{12} +\frac{3(1+\alpha)\alpha}{5}\ln\lambda+\alpha^2\ln\lambda+1\right)<0.978\ln\lambda+\frac{82}{7}.
	\end{align*}
	Plugging the last estimate into inequality \eqref{ratio_B}, we obtain $B(\tilde{u}+2\delta)/B(\tilde{u})=\omega(\lambda^{-0.98})$. Thus $F(u+\delta)/(F(u)-F(u+\delta))=\omega(\lambda^{-0.98})$ and  $F(u+\delta)=\omega(\lambda^{-0.98})$ which proves the statement in this lemma.
\qed\end{proof}

For $k < 7n/20$, we need a more careful analysis, where we will 
estimate the expected progress on fitness averaged over the random rates 
the algorithm may have at a time. Hence, we assume a fixed current fitness 
but a random current rate and compute the average drift of fitness with respect
 to the distribution on the rates.
%Hence, given a current fitness value, 
%we estimate the expected progress on fitness 
%for any fixed rate and the take the average value 
This approach is similar to the one by J{\"a}gersk{\"u}pper \cite{JaegerskuepperAlgorithmica11}, who computes 
the average drift of the Hamming distance to the optimum when the \ooea is 
optimizing a linear function, where the average is taken with respect to a distribution on all search points 
with a certain Hamming distance.

Of course, we want to exploit that a rate yielding 
near-optimal fitness progress 
is used most of the time such that too high (or too low) rates do not have 
a significant impact. To this end,  Lemma~\ref{lem:occupation} about occupation probabilities will be crucial.

We now define two fitness dependent bounds $r_l(k)$ and $r_u(k)$. We show in Lemma~\ref{expected-drift-far} that for any rate, if $r/F$ or $Fr$ is within the bounds, then the algorithm has logarithmic drift on fitness. 

\begin{definition}\label{define-r_l-r_u}
	Let %$n$ be sufficiently large, 
	$n/\ln\lambda<k<n/2$ and $F=32$. We define
	\begin{gather*}
		r_u(k):=\left\{
		\begin{array}{lll}
			n^2\ln(\lambda)/(12(n-2k)^2) &\text{~if~}& 7n/20\le k<n/2,\\
			10U(k)\ln(\lambda)/9 &\text{~if~}& n/\ln\lambda<k<7n/20.
		\end{array}
		\right.\\
		r_l(k):=\left\{
		\begin{array}{lll}
			L(k)\ln(\lambda)/2  &\text{~if~}& n/\ln\lambda \le k<n/2,\\
			F &\text{~if~}& n/\lambda<k<n/\ln\lambda.
		\end{array}
		\right.
	\end{gather*}
	where $L(k)$ and $U(k)$ are defined as in Definition \ref{define-L-U}.
\end{definition}
We notice that Lemma~\ref{far-drift-rate} can be applied to all $r>r_u$ or $r<r_l$ because for all $7n/20\le k<n/2$,
we have $r_u/(U(k)\ln\lambda)=22n/(12(2n-k))\ge22/(12(2-0.35))=10/9$. For $k<n/\ln\lambda$, we set $r_l$ to the minimal possible value of $r$. Finally note that $r_u$ is non-decreasing in $k$ due to the monotonicity of $n^2/(n-2k)^2$ and $U(k)$.
%We notice that Lemma~\ref{far-drift-rate} can be applied to all $r>r_u$ or $r<r_l$ because for all $7n/20\le k<n/2$,
%we have $r_u/(U(k)\ln\lambda)=20n/(11.95(2n-k))\ge20/(11.95(2-0.35))\ge 80/79$. For $k<n/\ln\lambda$, we set $r_l$ to the minimal possible value of $r$. Finally note that $r_u$ is non-decreasing in $k$ due to the monotonicity of $n^2/(n-2k)^2$ and $U(k)$.

%We denote by $k_t$ the number of ones of the current search point (often just called $k$ in the rest of the paper if the time index is irrelevant).

\begin{lemma}\label{expected-drift-far}
	Let %$n$ be sufficiently large, 
	$n/\lambda<k<n/2$ with $F=32$. 
	%Suppose that $\lambda\ge C\ln n$ for some sufficiently large constant~$C>0$ and $\lambda=o(\exp(\sqrt{n}))$. 
	Let $\Delta(k,r)$ denote the fitness gain of the best offspring using rate in $\{r/F,Fr\}$.
	\begin{enumerate}
		\item The negative drift of fitness for too high rates $ r\ge Fr_u$ is bounded by  \[
		E(\Delta(k,r))\ge -\big(1+o(1)\big)\frac{n-2k}{n}\frac{r}{F}.
		\]
		\item When $k\ge 7n/20$ the positive drift of fitness for good rate
		$r\le Fr_u$ is bounded by
		\[
		E(\Delta(k,r))\ge \big(1-o(1)\big)\cdot 10^{-4}\left(\frac{n-2k}{n}\cdot\frac{r}{F}+\min\bigg\{\ln\lambda,\frac{r}{F}\bigg\}\right).
		\]
		\item When $n/\lambda<k<7n/20$ the positive drift of fitness for good rate $r\le Fr_u$ is bounded by
		\[
		E(\Delta(k,r))\ge \big(1-o(1)\big)\min\bigg\{\frac{r}{F},\frac{\ln\lambda}{F\ln(en/k)}\bigg\}.
		\]
	\end{enumerate}
%	\begin{itemize}
%		\item If $r\ge Fr_u$, then $E(\Delta(k,r))\ge -(1+o(1))(n-2k)r/(Fn)$.
%		\item If $r\le Fr_u$ and $k\ge 7n/20$, then $E(\Delta(k,r))\ge (1-o(1))10^{-4}((n-2k)r/(Fn)+\min\{\ln\lambda,r/F\})$.
%		\item If $r\le Fr_u$ and $n/\lambda<k<7n/20$, then $E(\Delta(k,r))\ge (1-o(1))\min\{r,\ln(\lambda)/\ln(en/k)\}/F$.
%	\end{itemize}
\end{lemma}

\begin{proof}
	The probability of using rate $r/F$ is $1/2$. Thus with probability at least $1-(1/2)^{\lambda}=1-o(1/n^3)$, at least one offspring uses rate $r/F$. For this offspring, the expected loss is $(n-2k)r/(Fn)$. 
	If the complementary event (hereinafter called failure) of probability $o(1/n^3)$ happens, we estimate $\Delta(k,r)$ pessimistically by $-n$. This proves the first statement.
	
	To prove the second item, we take $i=10^{-4}((n-2k)r/(Fn)+\min\{\ln\lambda,r/F\})$.
	According to Lemma~\ref{far-prob}, the probability that an offspring uses rate $r/F$ and achieves progress of $i$ or more is at least $\lambda^{-0.98}/2$. Thus for $\lambda$ offspring, we obtain $\Pr(\Delta(k,r)\ge i)\ge 1-(1-\lambda^{-0.98}/2)^{\lambda}= 1-O(\exp(-\lambda^{0.02}/2))= 1-o(1)$. If the failure event happens, we estimate $\Delta(k,r)$ pessimistically by $-(n-2k)r/(Fn)=O(i)$. Thus the statement holds.
	
	For the third item, we take $i:=\min\{r,\ln(\lambda)/\ln(en/k)\}/F$. Notice that for $k<7n/20$ we have $r_u(k)<r_u(7n/20)=(25/27)\ln\lambda<0.93\ln\lambda$. Applying Lemma~\ref{lestimate}\ref{itestimate2} with $r/F\le r_u(k)=o(\sqrt{n})$ we obtain $(1-r/(Fn))^n \ge (1-o(1)) e^{-r/F}$. Therefore the probability that one offspring using rate $r/F<0.93\ln\lambda$ makes a progress of at least $i$ is lower bounded by (assuming $n$ large enough)
	\begin{align*}
		&\binom{k}{i}\left(\frac{r}{Fn}\right)^{i}\left(1-\frac{r}{Fn}\right)^n
		\ge \left(\frac{k}{i}\cdot\frac{r}{Fn}\right)^{i}\bigg((1-o(1))e^{-\frac{r}{F}}\bigg)\\
		&\quad>\left(\frac{k}{en}\right)^{i}e^{-0.94\ln\lambda}\ge\lambda^{-1/F-0.94}>\lambda^{-0.98}.
	\end{align*}
	Thus for $\lambda$ offspring, we obtain $\Pr(\Delta(k,r)\ge i)\ge 1-(1-\lambda^{-0.98}/2)^{\lambda}= 1-o(1/\ln(\lambda))$. If the failure event happens we estimate $\Delta(k,r)$ pessimistically by $-(n-2k)r/(Fn)=O(\ln\lambda)$. The contribution of failure events is $o(1)$ which is also $o(i)$. Therefore the third statement holds.
\qed\end{proof}

As discussed, our aim is to show that $r_t/F$ or $Fr_t$ stays in the right range frequently enough 
such that the overall average drift is still logarithmic.
We notice that small rates $r_t<r_l$ intuitively do not have a negative effect, 
therefore we focus on the probability that $r_t<Fr_u$. 
Since $r_u$ monotonically decreases when $k$ decreases, we need to analyze whether $r$ still stays in the right range
 if there are large jumps in fitness distance~$k$. Intuitively, the speed at which the mutation rate is decreased  
is much higher 
than than the decrease of fitness distance. To make this rigorous, we first 
look at the probability of large jumps, as 
detailed in the following lemma.

\begin{lemma}\label{lemma:prob-single-offspring}
	%Let $n$ be sufficiently large, $n/\ln\lambda<k<n/2$ and $F=32$. 
	Assume $r\le n/2$ and
	let $Z(k,r)$ denote the fitness-distance 
	increase when applying standard bit mutation with probability $p=r/n$ to an individual with $k$ ones. Then
	\begin{gather*}
		\Pr\left(Z(k,r)\le (n-2k)r/n-\Delta\right)
		\le\exp\left(\frac{-\Delta^2}{2(1-p)(r+\Delta/3)}\right),\\
		\Pr\left(Z(k,r)\ge (n-2k)r/n+\Delta\right)
		\le\exp\left(\frac{-\Delta^2}{2(1-p)(r+\Delta/3)}\right).
	\end{gather*}
\end{lemma}
\begin{proof}
	Without loss of generality, we assume that the individual has $k$ leading ones and $n-k$ trailing zeros. Let random variables $Z_1,\dots,Z_n$ be the contribution to fitness distance increase in each position after standard bit mutation. Then
	\begin{eqnarray*}
		&&\Pr(Z_i=-1)=p \text{ and } \Pr(Z_i=0)=1-p\text{ for all }1\le i\le k;\\
		&&\Pr(Z_i=1)=p \text{ and } \Pr(Z_i=0)=1-p\text{ for all }k< i\le n.
	\end{eqnarray*}
  The random variables $Z_1,\dots,Z_n$ are independent and $Z(k,r)=\sum_{i=1}^{n} Z_i$. Similarly as in the proof of Lemma~\ref{far-drift-rate} \ref{dec},
  we have $E(Z(k,r))=-kp+(n-2k)p=(n-2k)p$ and $\Var(Z(k,r))=\sum_{i=1}^{n} \Var(Z_i)=np(1-p)=(1-p)r$. To apply Bernstein's inequality (Theorem~\ref{tbernstein}), we construct $\tilde{Z_i}$ such that $\tilde{Z_i}=Z_i+p$ for all $1\le i\le k$ and $\tilde{Z_i}=Z_i-p$ for all $k<i\le n$. Therefore $E(\tilde{Z_i})=0$ and $\Var(\tilde{Z_i})=\Var(Z_i)$. 
	\begin{eqnarray*}
		&&\Pr(\tilde{Z_i}=-1+p)=p \text{ and } \Pr(\tilde{Z_i}=p)=1-p\text{ for all }1\le i\le k;\\
		&&\Pr(\tilde{Z_i}=1-p)=p \text{ and } \Pr(\tilde{Z_i}=-p)=1-p\text{ for all }k< i\le n.
	\end{eqnarray*}
	By assuming $r\le n/2$, we have $p\le 1/2$ and thus $p-1\le \tilde{Z_i}\le 1-p$ for all $1\le i\le n$. Using the fact that $\sum_{i=1}^{n}\tilde{Z_i}=Z(k,r)-E(Z(k,r))$,
	Theorem~\ref{tbernstein} yields with $b\coloneqq 1-p$ and $\sigma^2\coloneqq (1-p)pn=(1-p)r$ that 
	\[
	\Pr\left(\sum_{i=1}^{n}Z(k,r)-E(Z(k,r))\ge \Delta\right)\le\exp\left(\frac{-\Delta^2}{2(1-p)(r+\Delta/3)}\right).
	\]
	Similarly the lower tail bound holds.	
%	Let $X_1,\dots,X_k$ and $Y_1,\dots,Y_{n-k}$ be independent random variables representing the distance change such that $-X_i\sim\text{Bernoulli}(p)$ and $Y_i\sim  \text{Bernoulli}(p)$.  Thus $Z(k,r)\sim \sum_{i=1}^{k}X_i+\sum_{i=1}^{n-k}Y_i$ and $E(Z(k,r))=(n-2k)p$. We construct random variables $Z_1,\dots,Z_n$ such that $Z_i=X_i+p$ for $1\le i\le k$ and $Z_i=Y_{i-k}-p$ for $k<i\le n$. Therefore $Z_i\in\{-1+p,p\}$ for $1\le i\le k$ and $Z_i\in\{-p,1-p\}$ for $k<i\le n$. Since $p\le 1/2$ we obtain $p-1\le Z_i\le 1-p$. Using the fact that $E(Z_i)=0$ and $\Var(Z_i)=p(1-p)$, we apply Bernstein's inequality (Theorem~\ref{tbernstein}) to $\sum_{i=1}^{n}Z_i$ and obtain
%	\[
%	\Pr\bigg(\sum_{i=1}^{n}Z_i\ge \Delta\bigg)\le\exp\left(\frac{-\Delta^2}{2(np(1-p)+(1-p)\Delta/3)}\right).
%	\]
%	Since $Z(k,r)\sim \sum_{i=1}^{k}(Z_i-p)+\sum_{i=k+1}^{n}(Z_i+p)$ we obtain the upper tail bound of $Z(k,r)$. Similarly the lower tail bound holds.
\qed\end{proof}

We now use Lemma~\ref{lemma:prob-single-offspring} to show that once $r_t\ge Fr_u(k_t)$,
 there will be a strong drift for $r_t/r_u(k_t)$ to decrease down to $1$. 
\begin{lemma}\label{lem:far-rate-distance}
	Let %$n$ be sufficiently large, 
	$k_t<n/2$ and $F=32$. 
	Let $\tau:=\log_F(3/\sqrt{10})$ and $X_t:=\log_F(r_t/r_u(k_t))-\tau$ with $r_u(k_t)$ defined in Definition~\ref{define-r_l-r_u}, we have
	%Let $X_t=\log_F(r_t/r_u(k_t))$. If $k_t<n/2$, we have
	\begin{align*}
		&\Prob\left(X_{t+1}-X_t\ge a\mid X_t> 1 \right) \le \lambda^{-\Omega(a+1)} \text{ for all } a\ge -1/2,\\
		&\Prob\left(X_{t+1}-1\ge a\mid X_t\le 1\right) \le \lambda^{-\Omega(a+1)} \text{ for all } a>0.
	\end{align*}
\end{lemma}

\begin{proof}
	Using the fact that $r_{t+1}\in\{Fr_t,r_t/F\}$, we see that
	\[
	X_{t+1}-X_{t}\in\left\{1+\log_F\left(\frac{r_u(k_t)}{r_u(k_{t+1})}\right),-1+\log_F\left(\frac{r_u(k_t)}{r_u(k_{t+1})}\right)\right\}.
	\]
	According to the monotonicity that $r_u(k)$ increases with respect to $k$, we notice that $k_t\ge k_{t+1}$ is a necessary condition for $X_{t+1}-X_{t}\ge 1$. We also notice that $X_t\ge \tau$ is equivalent to $r_t/r_u(k_t)\ge 3/\sqrt{10}$, which is sufficient to apply Lemma~\ref{far-drift-rate}\ref{dec} since $r_u(k)\ge (10/9)U(k)\ln\lambda$ as defined in Definition~\ref{define-r_l-r_u}.
	
	We first consider the case $k_{t+1}\ge k_t$ (equivalent to $r_u(k_{t+1})\ge r_u(k_{t})$). In this case $X_{t+1}-X_t\le 1$ thus $\Pr(X_{t+1}\ge 1\;\cap\; k_{t+1}\ge k_t\mid X_t<0)=0$ and $\Pr(X_{t+1}-1\ge 1\;\cap\; k_{t+1}\ge k_t\mid X_t\le 1)=0$. It remains to consider
	\begin{align*}
		&\Pr(X_{t+1}-X_t\ge a\;\cap\; k_{t+1}\ge k_t\mid X_t>1)  \text{ with } -1/2\le a\le 1,\text{ and }\\
		&\Pr(X_{t+1}-1\ge a\;\cap\; k_{t+1}\ge k_t\mid 0\le X_t\le 1) \text{ with } 0<a<1.
	\end{align*}
	If $r_{t+1}=r_t/F$ then $X_{t+1}-X_t\le -1$. Clearly $X_{t+1}-X_t\ge a\ge -1/2$ is impossible. It also makes  $X_{t+1}\ge 1$ with $0\le X_t\le 1$ impossible. Thus, the two probabilities above are bounded by $\Prob(r_{t+1}=Fr_t\;\cap\; k_{t+1}\ge k_t\mid X_t\ge \tau)\le \Prob(r_{t+1}=Fr_t\mid X_t\ge \tau)=\lambda^{-\Omega(1)}$ according to Lemma~\ref{far-drift-rate}\ref{dec}.
	
	It remains to consider $k_{t+1}< k_t$ (equivalent to $r_u(k_{t+1})< r_u(k_{t})$). We make a case distinction based on the value of $(n-2k_t)^2$.  
	
	\textbf{Case 1:} $(n-2k_t)^2<2Fn\ln\lambda$. In this case, $r_u(k_t)=n^2\ln\lambda/(12(n-2k_t)^2)\ge n/(24F)$ which means that $X_t<1$ for all rates $r\le n/(2F)$. Thus $\Pr(X_{t+1}-X_t< a\;\cap\; k_{t+1}<k_t\mid X_t>1)=0$. When computing $\Pr(X_{t+1}-1\ge a\;\cap\; k_{t+1}< k_t\mid X_t\le 1)$, we notice that $X_{t+1}\ge 1+a$ 
	implies $\log_F((n/2F)/r_u(k_{t+1}))\ge 1+a+\tau$. Furthermore,
\[
\frac{n/2F}{r_u(k_{t+1})}=\frac{12(n-k_{t+1})^2}{(2F)n\ln\lambda}\ge F^{1+a+\tau}=\frac{3F^{1+a}}{\sqrt{10}} \text{ if and only if }
(n-k_{t+1})^2\ge \frac{16F^{1+a}n\ln\lambda}{\sqrt{10}}.
\]
Therefore a necessary condition for $X_{t+1}\ge 1+a$ while $X_t\le 1$ and $ (n-k_t)^2\le 2Fn\ln\lambda$ is $k_t-k_{t+1}\ge ((4F^{(1+a)/2}/10^{1/4}-\sqrt{2F})/2)\sqrt{n\ln\lambda}>(6F^{a/2}-4)\sqrt{n\ln\lambda}$. We notice that $E(k_{t+1}-k_{t})>0$, applying Lemma~\ref{lemma:prob-single-offspring}
%\cawi{It is not obvious how Lemma~\ref{lemma:prob-single-offspring} is applied. The lemma also involves the expected distance decrease before the $\Delta$ but here you talk about $k_t-k_{t+1}>\Delta$ only.} 
 and using a union bound we obtain for $\Delta:=(6F^{a/2}-4)\sqrt{n\ln\lambda}>2\sqrt{n\ln\lambda}$ that
\begin{align*}
\Pr\left(k_t-k_{t+1}>\Delta\mid X_t\le 1\right)&=\Pr\left(k_{t+1}-k_t<-\Delta\mid X_t\le 1\right)\\
&<\Pr\left(k_{t+1}-k_t<\E(k_{t+1}-k_t)-\Delta\mid X_t\le 1\right)\\
&<\lambda\exp\left(\frac{-\Delta^2}{2(n/2+\Delta/3)}\right)<\lambda\exp\left(\frac{-\Delta^2}{n+\Delta}\right)=\lambda^{-\Omega(1+a)}.
\end{align*}
Therefore $\Pr(X_{t+1}-1\ge a\;\cap\; k_{t+1}<k_t\mid X_t\le 1)=\lambda^{-\Omega(1+a)}$.

\textbf{Case 2:} $(n-2k_t)^2\ge 2Fn\ln\lambda$. Let 
\[
\sigma^2_t\coloneqq r_u(k_{t})/r_u(k_{t+1})=(n-2k_{t+1})^2/(n-2k_{t})^2,
\]
then $X_{t+1}-X_t\in\{1+\log_F(\sigma^2_t),-1+\log_F(\sigma^2_t)\}$.
We rewrite for $X_t>1$ and $a\ge -1/2$
\begin{align}
&\Prob\left(X_{t+1}-X_t\ge a\;\cap\; k_{t+1}<k_t\mid X_t\right)\nonumber\\
\le& \Prob\left(r_{t+1}=r_t/F\;\cap\; \sigma^2_t\ge F^{a+1}\mid X_t\right)+\Prob\left(r_{t+1}=Fr_t\;\cap\; \sigma^2_t\ge F^{a-1}\mid X_t\right)\nonumber\\
\le& \Prob\left(\sigma^2_t\ge F^{a+1}\mid X_t\right)+\Prob\left(\sigma^2_t\ge F^{a-1}\mid X_t\right)\mathds{1}_{a>2}+\Prob\left(r_{t+1}=Fr_t\mid X_t\right)\mathds{1}_{a\le 2},\label{eq5}
%\\&\le \Prob\left(\sigma^2_t\ge F^{a+1}\mid X_t>1\right)+\Prob\left(\sigma^2_t\ge F^{a-1}\mid X_t>1\right)\mathds{1}_{a>2}+\lambda^{-\Omega(1)}\mathds{1}_{a\le 2}
\end{align}
as well as for $X_t\le 1$ and $a>0$
\begin{align}
&\Prob\left(X_{t+1}-1\ge a\;\cap\; k_{t+1}<k_t\mid X_t\right)=\Prob\left(X_{t+1}-X_t\ge 1+a-X_t\;\cap\; k_{t+1}<k_t\mid X_t \right)\nonumber\\
\le& \Prob\left(r_{t+1}=r_t/F\;\cap\;\sigma^2_t\ge F^{a+2-X_t}\mid X_t\right)+\Prob\left(r_{t+1}=Fr_t\;\cap\;\sigma^2_t\ge F^{a-X_t}\mid X_t\right)\nonumber\\
\le& \Prob\left(\sigma^2_t\ge F^{a+1}\mid X_t\right)+\Prob\left(r_{t+1}=Fr_t\;\cap\;\sigma^2_t\ge F^{a-X_t}\mid X_t\right),\label{eq6}
%&\le \Prob\left(\sigma^2_t\ge F^{a+1}\mid X_t\right)+\Prob\left(\sigma^2_t\ge F^{a-X_t}\mid X_t\right)\mathds{1}_{X_t<\tau \lor a>2}+\Prob\left(\frac{r_{t+1}}{r_t}=F\mid X_t\right)\mathds{1}_{X_t\ge\tau\land a\le 2}.
\end{align}
where the second item in the above inequality \eqref{eq6} is furthermore bounded in \eqref{eq7} by  making a distinction between  $X_t\ge \tau \land a\le 2$ and the remaining cases.
\begin{align}
&\Prob\left(r_{t+1}=Fr_t\;\cap\;\sigma^2_t\ge F^{a-X_t}\mid X_t\right)\nonumber \\
\le& \Prob\left(\sigma^2_t\ge F^{a-X_t}\mid X_t\right)\mathds{1}_{X_t<\tau \lor a>2}+\Prob\left(r_{t+1}=Fr_t\mid X_t\right)\mathds{1}_{X_t\ge\tau\land a\le 2}.\label{eq7}
\end{align}
%\cawi{I would give the final line of the two displayed formulas a number. The following argumentation ("`it remains to consider'') is too short. Specify precisely which of the terms from the two estimations (e.\,g. $\Prob\left(\frac{r_{t+1}}{r_t}=F\mid X_t\right)\mathds{1}_{a\le 2}$) you estimate by $\lambda{-\Omega(1)}$. If I understand it correctly, only the middle terms from the two displayed formulas are hard to estimate. However, one of these has a $F^{a-1}$ and the other one a $F^{a-X_t}$ on the right-hand side. Am I right in that you subsume both under the $\Prob(\sigma^2_t\ge F^{a-\tau})$?  This should be explained. Also the $\tau$ appears somewhat unexpectedly -- it seems to stem from the definition of $X_t$ that you expand here again. Altogether, several additional intermediate steps are needed before the following paragraph becomes sufficiently clear, I think.}
Applying Lemma~\ref{far-drift-rate}\ref{dec} we see that both $\Prob\left(r_{t+1}=Fr_t\mid X_t\right)\mathds{1}_{a\le 2}$ from \eqref{eq5} and $\Prob\left(r_{t+1}=Fr_t\mid X_t\right)\mathds{1}_{X_t\ge\tau\land a\le 2}$ from \eqref{eq7} are of order $\lambda^{-\Omega(1)}$. This $\Omega(1)$ exponent 
%\cawi{exponent?}
is sufficient to prove the lemma for $a\le 2$. We also notice that the event $\sigma^2_t\ge F^{a-\tau}$ subsumes all the other remaining events in inequalities \eqref{eq5}, \eqref{eq6}, and \eqref{eq7}. 
%\cawi{Do you really mean $\sigma^2\ge -\tau$? Or $\sigma^2_t\ge F^{a-\tau}$? }.  \jing{it is a typo. it should be $\sigma^2\ge F^{-\tau}$}\cawi{Alright. Please remove the comment.}
%For the remaining items in inequalities \eqref{eq5}, \eqref{eq6}, and \eqref{eq7}, we notice that $a+1\ge 1/2$ for $a\ge -1/2$ in \eqref{eq5}, $a-1> 1$ for $a>2$ in \eqref{eq5}; $a+1> 1$ for $a> 0$ in \eqref{eq6};  $a-X_t> -\tau$ for $X_t<t$ and $a>0$ in \eqref{eq7}. The most critical item \cawi{A bit sloppy. Can we say: Hence, the event $\sigma^2_t\ge F^{a-\tau}$ subsumes all the other remaining events?} among these is $\Pr(\sigma^2\ge -\tau\mid X_t)$ \cawi{Do you really mean $\sigma^2\ge -\tau$? Or $\sigma^2_t\ge F^{a-\tau}$? }.  
Therefore it remains to validate $\Prob\left(\sigma^2_t\ge F^{a-\tau}\mid X_t\right)\le\lambda^{-\Omega(a+1)}$ for $a\ge 0$. To ease representation, let $s:=F^{(a-\tau)/2}-1\ge F^{-\tau/2}-1=(10/9)^{1/4}-1>1/40$.
Since $s=\Omega(1+a)$, proving $\Prob(\sigma_t\ge 1+s\mid X_t)=O\left(\lambda^{-\Omega(s)}\right)$ is sufficient to conclude the analysis of this case and therefore the lemma.
We rewrite
\begin{align*}
\Prob\left(\sigma_t\ge 1+s\mid X_t\right)&=\Prob\left(\frac{n-2k_{t+1}}{n-2k_{t}}\ge 1+s\mid X_t\right)\\
&=\Pr\left(k_t-k_{t+1}\ge s(n-2k_t)/2\mid X_t\right).
\end{align*}
Let $\Delta:=(s/2+p)(n-2k_t)$ for $0<p\le 1/2$.
Applying Lemma~\ref{lemma:prob-single-offspring} and using a union bound we obtain
\begin{align*}
&\Pr\left(\sigma_t\ge 1+s\mid X_t\right)<\lambda\exp\left(\max_{0<p\le 1/2}\left\{\frac{-\Delta^2}{2(1-p)(pn+\Delta/3)}\right\}\right)\\
<&\lambda\exp\left(\max_{0<p\le 1/2}\left\{\frac{-\Delta}{2(1+1/3)}\mathds{1}_{pn\le \Delta}+\frac{-\Delta^2}{2(1-p)(pn)(1+1/3)}\mathds{1}_{pn>\Delta}\right\}\right)\\
<&\lambda\exp\left(-\min_{0<p\le 1/2}\left\{\frac{\Delta}{3}\mathds{1}_{pn\le \Delta}+\frac{\Delta^2}{3(1-p)(pn)}\mathds{1}_{pn>\Delta}\right\}\right)
\end{align*}
We notice that $\Delta\ge (s/2)\sqrt{2Fn\ln(\lambda)}=4s\sqrt{n\ln(\lambda)}$ and $(s/2+p)^2/((1-p)p)$ attains the minimal value $s(2+s)>2s$ when $p=s/(2(s+1))$. Using the fact that $(n-2k_t)^2/n\ge 2F\ln(\lambda)$ and $s> 1/40$,
\begin{align*}
&\Pr\left(\sigma_t\ge 1+s\mid X_t\right)
<\lambda\exp\left(-\min\left\{s\sqrt{n\ln\lambda}\mathds{1}_{pn\le \Delta}+\frac{(2s)2F\ln\lambda}{3}\mathds{1}_{pn>\Delta}\right\}\right)\\
&<\lambda\exp\left(-\min\left\{s\sqrt{n\ln(\lambda)}\mathds{1}_{pn\le \Delta}+42s\ln(\lambda)\mathds{1}_{pn>\Delta}\right\}\right)=\lambda^{-\Omega(s)}.
\end{align*}
\qed\end{proof}

We finally use Lemma~\ref{lem:far-rate-distance} and Lemma~\ref{lem:occupation} 
to obtain a logarithmic drift on average. After this major effort, it is a matter of a 
relatively straightforward drift analysis of fitness distance to obtain the following bound 
on the time to leave the far region. 
 
\begin{theorem}
\label{theo:runtime-far}
%Let $n$ be sufficiently large. 
%Let $\lambda\ge C\ln n$ for a sufficiently large constant $C>0$ and $\lambda=o(\exp(\sqrt{n}))$. 
The \oclea with self-adapting mutation rate reaches a \onemax-value of $k\le n/\lambda$ within an expected number of $O(n/\log\lambda)$ iterations, regardless of the initial mutation rate.
Furthermore, with probability at least $1-o(1)$, it holds  $k_{t'}\le 2n/\lambda$ and $r_{t'}\le (7/9)\ln\lambda$ for some $t'=O(n/\log\lambda)$. 
	%For any initial search point $k_0$ and $r_0$, the hitting time~$T$ for $k_{t}\le n/\lambda$ has expectation $\E(T)=O( n/\log\lambda)$. Moreover with probability at least $1-o(n^{-1.2})$, it holds  $k_{t'}\le 2n/\lambda$ and $r_{t'}\le (7/9)\ln\lambda$ for some $t'=O(n/\log\lambda)$. \jing{I think $1-o(1)$ is enough for the main theorem?}
\end{theorem}
\begin{proof}
	We first argue that within an expected number of $O(\sqrt{n})$ generations we will have $k_t<n/2$. Consider the case that $k_t\ge n/2$ and  let the independent random variables $X$ and $Y$ denote the number of flips in $k_t$ one-bits and $(n-k_t)$ zero-bits, respectively, in an offspring using rate $p=r/n$. Referring to \cite{Doerr18bookchapter} for $p\in[2/n,1/2]$ we obtain, using  similar arguments in the proof Lemma~\ref{far-drift-rate}\ref{dec} that $\Pr(X\ge \E(X)+1)=\Theta(1)$ and $\Pr(Y\le \E(Y))=\Theta(1)$. Then $\Pr(X-Y\ge \E(X)-\E(Y)+1)=\Theta(1)$. Since $\E(X)\ge E(Y)$, the probability that an offspring choose rates $\tilde{r}\in\{r_t/F,Fr_t\}$ with $2\le \tilde{r}\le n/2$ and have $X-Y\ge 1$ is at least $1/2\cdot\Theta(1)=\Theta(1)$. Since the best of $\lambda=\Omega(\ln n)$ offspring is selected, the probability that $k_{t+1}\le k_t-1$ holds is at least $1-\exp(-\Theta(\lambda))=1-o(1/n^2)$.  By an additive drift theorem, it takes $O(\max\{k_0-n/2,0\})=O(\sqrt{n})$ iterations from the initial random search point to reach a parent with fitness distance less than $n/2$.
	 %\cawi{I cannot follow. It seems you conclude that the positive progress also happens with constant probability, but why is that.} \jing{ the probability of making at least $1$ step progress is a constant}\cawi{Now I see. I reformulated the paragraph a little bit. Please check.}

	Without loss of generality, we can now assume $k_0<n/2$. Consider the number of one-bits flips $X$ and zero-bits flips $Y$ in a parent with fitness distance $k_t<n/2$ and rate $2\le r<n/2$. As argued above $\Pr(X-Y\ge E(X)-E(Y)+1)=\Theta(1)$. Since $k_t-(\E(X)-\E(Y))=k_t-(k_t-(n-k_t))r/n=k_t(1-r/n)+(n-k_t)(r/n)<n/2$ for all $r<n/2$, the probability that an offspring has fitness distance at most $n/2-1$ is $\Theta(1)$. Thus for $\lambda=\Omega(\ln n)$ offspring, we have $\Pr(k_{t+1}<n/2)\ge 1-\exp(-\Theta(\lambda))=1-o(1/n^2)$. Since that we aim at proving a hitting time of $O(n/\ln \lambda)$ and only 
	consider phases of this length, we may furthermore assume $k_t<n/2$ for all $t\ge 0$, which only introduces an $o(1)$ error term 
	by a union bound. %\cawi{Please check the last argument, which I added for clarity.}
	
	Define random variables $X_t:=\log_F(r_t/r_u(k_t))-\tau$ with $\tau=\log_F(3/\sqrt{10})<0$. We notice that when $(n-2k_t)^2\le 2Fn\ln\lambda$ we have $X_t<1$. If $r_t\ge Fr_u(k_t)$, according to Lemma~\ref{far-drift-rate}\ref{dec}, with probability $1-o(1)$ we have $r_{t-1}=r_t/F$. Therefore within $O(\ln n)$ iterations we will obtain $X_t\le 1$.

	The idea of the remaining proof is to compute an average drift for any fixed distance using the distribution of mutation rates, and then to apply the variable drift theorem to obtain a runtime bound. Applying Lemma~\ref{lem:far-rate-distance} and Lemma~\ref{lem:occupation} to the $X_t$, we see that
	\[
	 \Pr(r_t\ge F^{1+a+\tau}r_u(k_t))\le \lambda^{-\Omega(a)} \text{ for all } a>0.
	\] 
	Let $r^{(i)},i\in\Z,$ denote the rate between $(F^{i}r_u(k),F^{i+1}r_u(k)]$ corresponding to fitness distance~$k$. 
	Thus, for all $i\ge 1$, we obtain 
	\[
	\Pr\bigg(r^{(i)}\bigg)\le \Pr\bigg(r_t>F^{i}r_u(k_t)\bigg)\le \Pr\bigg(r_t\ge F^{1+(i-1-\tau)+\tau}r_u(k_t)\bigg)\le \lambda^{-\Omega(i-1-\tau)}.
	\]	
	According to Lemma~\ref{expected-drift-far},  $E(\Delta(k,r^{(i)}))\ge-(1+o(1))(n-2k)r^{(i)}/n)$ for $i\ge 1$ and $E(\Delta(k,r^{(0)}))\ge \Omega((n-2k)r^{(0)}/n)$. The contribution of the negative drift is a lower order term compared to the contribution of the positive drift. Let $\Delta^{(k)}$ denote the average drift at distance $k$. We obtain
	\begin{align*}
	\Delta^{(k)}&=\sum_{i\in\Z}E(\Delta(k,r^{(i)}))\Pr(r^{(i)})\ge (1-o(1))\sum_{i\le 0}E(\Delta(k,r^{(i)}))\Pr(r^{(i)}).
	\end{align*}
	We notice that $\sum_{i\le 0}\Pr(r^{(i)})=1-o(1)$ and  $E(\Delta(k,r^{(i)}))>0$ for all $i\le 0$. According to Lemma~\ref{far-drift-rate}\ref{inc}, with at least constant probability $r_t=\Omega(r_l(k_t))$. 	%\cawi{How do we know that without a statement on the distribution of the rate over time? It does not immediately follow 	from Lemma~\ref{far-drift-rate}} \jing{It is similar to what we did last year. Since small mutation rate does not harm the fitness distance, we don't need to make this $\Pr(r<r_l)$ as precise as $\Pr(r>r_u)$ }\cawi{Now it is clear to me.} 
	Since for any rate $r=\Omega(r_l(k))$ and $r\le Fr_u$ the drift is $\E(\Delta(k,r))\ge \Theta(\ln(\lambda)/\ln(n/k_t))$ according to Lemma~\ref{expected-drift-far}, the average drift satisfies
	\begin{align*}
	\Delta^{(k)} \ge \Theta(\ln(\lambda)/\ln(n/k)).
	\end{align*}
%	According to Lemma~\ref{far-drift-rate}, we can easily see \cawi{How do we know that without a statement on the distribution of the rate over time? It does not immediately follow 
%	from Lemma~\ref{far-drift-rate}} 
%	\jing{arg1}
%	that $\Pr(r_l(k)\le r\le r_u(k)\mid k)=\Theta(1)$.
%	Therefore the average drift is
%	\begin{align*}
%	\Delta^{(k)} & \ge \big(1-o(1)\big)\min_{r_l(k)\le r<Fr_u(k)}\big\{E(\Delta(k,r))\big\}\Pr\big(r_l(k)\le r<Fr_u(k)\big) \\ 
%	& \quad =\Theta(\ln(\lambda)/\ln(n/k)).
%	\end{align*}
	Using the variable drift theorem (Theorem~\ref{theo:variable-upper}) and the fact that
	\begin{align*}
		&\int_{n/\lambda}^{n/2} \frac{\ln(n/k)}{\ln(\lambda)}\,\mathrm{d}k=	\frac{\big(k\ln(n)-k\ln(k)+k\big)\big|_{n/\lambda}^{n/2}}{\ln\lambda}
		=\frac{\Theta(n)}{\ln\lambda},
	\end{align*}
	the expected time to reduce the fitness distance to at most $n/\lambda$ conditioning on the assumption that $k_t<n/2$ for some $t=O(\sqrt{n})$ and $k_{t'}<n/2$ for all $t\le t'=O(n/\log\lambda)$ is then $\Theta(n/\log\lambda)$. Thus the runtime bound of $O(n/\log\lambda)$ holds with probability $\Omega(1)$ due to Markov's inequality. Using a restart argument we then obtain the claimed expected runtime since the expected number of repetition of a phase of length $O(n/\log\lambda)$ is $O(1)$. 
	%\cawi{I think we must say that the time $O(n/\log\lambda)$ holds with probability $\Omega(1)$ due to Markov's inequality. Hence, a phase of length $O(n/\log\lambda)$ is successfull with prob. $(1-o(1))\Omega(1)=\Omega(1)$.}
  %\cawi{This is the expected time conditioned on that some failure events (e.g. the ones from the first paragraph do not occur. However, the theorem makes a claim on the unconditional expected time. Don't we need an additional argument of the kind ``if anything goes wrong, we repeat the argumentation; the expected number of repetitions is $O(1)$''?}
	
	To prove the second statement of the theorem, we notice that the corresponding upper bound on the rate for $k_t=o(n)$ is $r_u(k_t)\le (10/9)(U(k_t))\ln\lambda=((10/9)(2/22)+o(1))\ln\lambda <(1/9)\ln\lambda$ and the occupation probability satisfies $\Pr(r_t\le (7/9)F\ln\lambda\mid k_t=o(n))\ge 1-\lambda^{-\Omega(1)}=1-o(1)$. Therefore with probability $1-o(1)$, the first iteration such that  $k_t\le 2n/\lambda$ has rate $r_t\le (7/9)\ln\lambda$. We then argue for this iteration that with high probability it satisfies $r_{t+1}=r_t/F$ and $k_{t+1}\le2n/\lambda$. The probability of being no worse than parent using mutation probability $p\le (7/9)\ln(\lambda)/n$ is at least $(1-p)^n\ge (1-o(1))\lambda^{-7/9}>\lambda^{-8/9}$. Therefore,
	\[
	\Pr(k_{t+1}\le k_t \mid r_t\le (7/9)F\ln\lambda )\ge 1-\left(1-\lambda^{-8/9}/2\right)^{\lambda}=1-o(1).
	\]
	Furthermore $\Pr(r_{t+1}=Fr_t\mid k_t=o(n),r_t\ge (7/9)\ln\lambda )\le \lambda^{1-(23/22)7}=o(1)$. Then we obtain an iteration with $k_t\le 2n/\lambda$ and $r_t\le (7/9)\ln\lambda$ with probability $1-o(1)$.
\qed\end{proof}

%%% NEAR REGION

\subsection{The Near Region}
\label{sec:near}

We now analyze the regime in which the fitness distance satisfies $k = k_t=O(n/\lambda)$, the so-called 
near region. In this region, the probability that a fixed offspring created with rate $r$ is better than its parent is $\Theta(\frac 1\lambda \frac{r}{e^r})$, see Lemma~\ref{p_0-and-p_-}. Consequently, the probability to make progress is only $\Theta(\frac{r}{e^r})$. This implies that the optimal rate $r$ is constant (and by taking care of the constants, we shall see that the optimal rate value for the parent is $r = F$, the minimal possible value). 

The superiority of small rate values is sufficiently strong to show that the rate drifts towards these values (Lemmas~\ref{lem:drift-rate-near} and~\ref{lem:occprobnear}), however, for small values of $\lambda$ we cannot show that in this regime, which takes at least an expected number of $\Omega(n/\lambda)$ iterations, it never happens that the rate increases to a value which lets all offspring be worse than the parent (this happens from $r \ge C\lambda$ for a suitable constant $C$ on). Consequently, we cannot exclude the possibility that the algorithms loses fitness occasionally.

To analyze the progress of the algorithm (proof of Lemma~\ref{lem:nearmain}), we devise a potential function based on the current fitness and rate and show that the expected progress with respect to this potential is high enough. This allows to use the multiplicative drift theorem to argue that within a desired time, we reach the optimum.

More a technical issue is that, naturally, we also have to argue that the process does not leave the near region except with small probability. This is done in Lemma~\ref{lem:goback}.
%
%Informally speaking, this region  
%is responsible for the $(n\log n)/\lambda$ term in the expected number of generations since 
%here the probability  of an improvement is only $O(1/\lambda)$ and the offspring population 
%can boost this probability by a factor of $\Theta(\lambda)$, assuming constant rate.

We start with determining the probability of making progress in one mutation and similar events.

\begin{lemma}\label{p_0-and-p_-}
	Let %$n$ be sufficiently large, 
	$0<k\le 3n/\lambda$, and $r=o(\lambda^{1/4})$.	Let $x \in \{0,1\}^n$ with fitness distance $f(x) = k$. Let $y \in \{0,1\}$ be obtained from $x$ by flipping each bit independently with probability $r/n$. Consider the  probabilities
	\begin{align*}
	p_-(r) &:= \Pr(f(y) < f(x)),\\
	p_0(r) &:= \Pr(f(y) = f(x)),\\
	p'(r) &:= \Pr(\forall i \in [1..n] : x_i = 0 \implies y_i = 0),
  \end{align*}
  that is, the probabilities that the offspring is better than the parent, that is is equally good, and that none of the $0$-bits of the parent were flipped in the generation of the offspring. 
  
  Then
	\begin{gather*}
		(1-o(1))\tfrac{kr}{n}e^{-r} < p_-(r) < (1+o(1))\tfrac{kr}{n}e^{-r},\\
		(1-o(1))e^{-r}<p_0(r)<(1+o(1))e^{-r},\\
		(1-o(1))e^{-r}<p'(r)<(1+o(1))e^{-r}.
	\end{gather*}
\end{lemma}

\begin{proof}
	We regard the number $X$ of flips in the $k$ one-bits (``good flips'' which reduce the fitness distance) and the number $Y$ of flips in the $(n-k)$ zero-bits of the parent (``bad flips'' which increase the fitness distance). 
	Then $p_-(r)$ is at least 
	\[p_-(r) \ge \Pr(X=1,Y=0)=\frac{kr}{n}\left(1-\frac{r}{n}\right)^{n-1} \ge (1-o(1))\frac{kr}{n}e^{-r},\] 
	where the last estimate uses Lemma~\ref{lestimate}~\ref{itestimate2}.
	
	Since $r = o(\lambda^{1/4})$, we have $kr/n=o(1)$, $kr^2/n=o(1)$, and $(kr^2/n)^{1.5}=o(kr/n)$. This allows to bound $p_-(r)$ from above by
	\begin{align*}
		p_-(r) &<  \Pr(X\in\{1,2\},Y=0)+\sum_{i=3}^{2k-1}\Pr(X+Y=i,X>Y)\\
		& < \frac{kr}{n}\left(1-\frac{r}{n}\right)^{n-1}+\frac{k^2r^2}{2n^2}\left(1-\frac{r}{n}\right)^{n-2}\\
		& \quad + \sum_{i=3}^{2k-1}(i-1)\left(\frac{r}{n}\right)^{i}\left(1-\frac{r}{n}\right)^{n-i}
		\binom{k}{\lceil i/2\rceil}\binom{n-k}{\lfloor i/2\rfloor}\\
    & < \frac{kr}{n}\left(1-\frac{r}{n}\right)^{n-2}\left(1-\frac{r}{n}+\frac{kr}{2n}\right) + \sum_{i=3}^{2k-1} \left(\frac{r}{n}\right)^{i}\left(1-\frac{r}{n}\right)^{n-i}(kn)^{i/2}\\
%		&<\frac{kr}{n}\left(1-\frac{r}{n}\right)^{n-1}+\frac{k^2r^2}{n^2}\left(1-\frac{r}{n}\right)^{n-2}+\sum_{i=3}^{2k-1}\left(\frac{r}{n}\right)^{i}\left(1-\frac{r}{n}\right)^{n-i}(kn)^{i/2}\\
		& < (1+o(1))\frac{kr}{n}\left(1-\frac{r}{n}\right)^{n}+\sum_{i=3}^{2k-1}\left(\frac{kr^2}{n}\right)^{i/2}\left(1-\frac{r}{n}\right)^{n-i}\\
		& < (1+o(1))\frac{kr}{n}e^{-r}.
	\end{align*}
	%	\begin{align*}
	%	p_-(r)&<\Pr(X=1,Y=0)+\Pr(X=2)+\Pr(X=3)+\cdots\\
	%	&<\frac{kr}{n}\left(1-\frac{r}{n}\right)^{n-1}+\binom{k}{2}\left(\frac{r}{n}\right)^{2}+\binom{k}{3}\left(\frac{r}{n}\right)^{3}+\cdots\\
	%	&<(1+o(1))\frac{kre^{-r}}{n}+\frac{k^2r^2}{2n^2}+\frac{k^3r^3}{6n^3}<\frac{kr}{n}\left((1+o(1))e^{-r}+\frac{kr}{n}\right).
	%	\end{align*}
	Similarly for $p_0(r)$ we have
	\[
	p_0(r)>\Pr(X=Y=0)=\left(1-\frac{r}{n}\right)^{n} \ge (1-o(1))e^{-r}.
	\]
	Using again the fact that $kr^2/n=o(1)$, we have
	\begin{align*}
		p_0(r)&=\Pr(X=Y=0)+\sum_{i=1}^{k}\Pr(X=Y=i)\\
%		&=\left(1-\frac{r}{n}\right)^{n}+\binom{k}{1}\binom{n-k}{1}\left(\frac{r}{n}\right)^{2}\left(1-\frac{r}{n}\right)^{n-2}+\binom{k}{2}\binom{n-k}{2}\left(\frac{r}{n}\right)^{4}\left(1-\frac{r}{n}\right)^{n-4}+\cdots\\
		&=\left(1-\frac{r}{n}\right)^{n}+\sum_{i=1}^{k}\binom{k}{i}\binom{n-k}{i}\left(\frac{r}{n}\right)^{2i}\left(1-\frac{r}{n}\right)^{n-2i}\\
		&<e^{-r}+\sum_{i=1}^{k}\left(\frac{kr^{2}}{n}\right)^ie^{-r}<(1+o(1))e^{-r}.
%		&<e^{-r}+\frac{r^2e^{-r}}{\lambda/2}+\frac{r^4e^{-r}}{(\lambda/2)^2}+\cdots<(1+o(1))e^{-r}.
	\qedhere\end{align*}
	
	Finally, for $p'(r)$ we compute $p'(r) = \Pr(Y = 0) = (1 - \frac rn)^{n-k} = (1 \pm o(1)) e^{-r}$.
\qed\end{proof}

\begin{lemma}\label{lem:noback}
  Consider one iteration of the self-adaptive \oclea starting with an individual of fitness distance $k$ and rate $r = o(\lambda^{-1/4})$. Then the probability that there is an offspring which uses rate $r/F$ and which inherits all $0$-bits from the parent (and thus is at least as good as the parent), is at least $1 - \exp(-\tfrac 12 \lambda (1-o(1)) e^{-r/F})$.
\end{lemma}

\begin{proof}
  We compute 
  \[1 - (1 - \tfrac 12 p'(\tfrac rF))^\lambda \ge 1 - (1 - \tfrac 12 (1-o(1)) e^{-r/F})^\lambda \ge 1 - \exp(-\tfrac 12 \lambda (1-o(1)) e^{-r/F}).\]
\qed\end{proof}

The following lemma is the counterpart of Lemma~\ref{far-drift-rate}~\ref{itfardrifttoolarge}, where now 
the optimal rate is the smallest possible value~$F$. Again, we regard the event that all best offspring are created with the higher rate, since---due to our tie-breaking rule---only this leads to an increase of the rate. Different from  Lemma~\ref{far-drift-rate}~\ref{itfardrifttoolarge},  
now the probability of making a rate-increasing step is no $o(1)$ in general. If 
$k_t=\Theta(n/\lambda)$ and $r_t=O(1)$, we still have a small constant probability 
of increasing the rate. 
%benjamin: informal argument was not correct, I think. which is due to the fact that with constant probability all offspring  copy   the parent.

\begin{lemma}\label{lem:drift-rate-near}
	Let %$n$ be sufficiently large, 
	$0<k\le 3n/\lambda$ and $F=32$. The probability that all best offspring have been created with rate $Fr$ is at most $(1+o(1)) \frac{\lambda k Fr}{n} e^{-Fr}$ when $r<\ln\lambda$ and it is at most $\exp(-9r)$ for all~$r$.
\end{lemma}

\begin{proof}
  Let first $r<\ln\lambda$. According to Lemma~\ref{p_0-and-p_-},
	\begin{gather*}
		p_-(Fr)\le (1+o(1))\frac{Fkr}{n} e^{-Fr} \text{\quad and \quad} p_0(r/F)\ge (1-o(1))e^{-r/F}.
	\end{gather*}
	Therefore with probability at least $1-\lambda p_-(Fr) = 1 - (1+o(1)) \frac{\lambda Fkr}{n} e^{-Fr}$, no offspring of rate $Fr$ is better than its parent. Furthermore, by Lemma~\ref{lem:noback}, with probability at most $\exp(-(1-o(1))\tfrac 12 \lambda \exp(-r/F)) \le \exp(-(1-o(1))\tfrac 12 \lambda^{1-1/F})$ there is no offspring using rate $r/F$ and being equally good as its parent.
	Hence, the probability that a best offspring has been created with rate $r/F$ is more than
	\[
	1 - (1+o(1)) \frac{\lambda Fkr}{n} e^{-Fr} - \exp(-(1-o(1))\tfrac 12 \lambda^{1-1/F}) > 1 - (1+o(1)) \frac{\lambda Fkr}{n} e^{-Fr}.
	\]

	Note that for $r < \ln \lambda$, the second bound follows from the first. If $r\ge \ln\lambda$, then the second bound follows from applying Lemma~\ref{far-drift-rate} to $U(k)=1/11+o(1)$. 
\qed\end{proof}

We shall use the lemma above twice, first to bound the probability to have a certain rate (which will be needed to estimate the negative fitness drift) and second to estimate that a suitable two-dimensional drift is of the right order. We start with the occupation probability argument for the rate values. 

\begin{lemma}\label{lem:occprobnear}
  Consider a run of the self-adaptive \oclea started with some search point of fitness distance $k_0 \le 2 n / \lambda$ and rate $r_0 = F$. While the current search point of the algorithm has a fitness distance of at most $3 n / \lambda$, the probability that the current rate is $F^i$ is at most $\exp(-8 F^{i-1})$ for all $i \in \N_{\ge 2}$.
\end{lemma}

\begin{proof}
  If the current search point has fitness distance at most $3 n / \lambda$ and the current rate is $r$, then by Lemma~\ref{lem:drift-rate-near} the rate in the next iteration is $Fr$ with probability at most $\exp(-9r)$; note that this estimate is not affected by a possible cap of the rate at $\rmax$. 
  
  Consequently, the random process describing the rates is such that from rate $F^i$, $i \in [1..\log_F(\rmax)]$, we go to rate $F^{i+1}$ with probability at most $p_i = \exp(-9F^{i})$. Otherwise, we go to rate $F^{i-1}$ if $i \ge 2$ and stay at rate $F$ if $i=1$. By Lemma~\ref{lem:occsimple}, note that we obviously have $p_i / (1-p_{i}) \le p_{i-1}$, 
  in each iteration (such that the fitness distance has never gone above $3n/\lambda$) and for each $i \ge 2$ the probability $q_i$ that the current rate is $F^i$ is at most
  \[q_i \le \prod_{j=1}^{i-1} \frac{p_j}{1-p_j} \le \frac{p_{i-1}}{1-p_{i-1}} \le \exp(-8F^{i-1}).\]
\qed\end{proof} 

We use these occupation probabilities to estimate the drift away from the optimum (``negative drift''). From this we derive the statement that with high probability, the fitness distance does not increase to above $3n/\lambda$ in $n\lambda$ iterations.

\begin{lemma}\label{lem:goback}
  In the situation of Lemma~\ref{lem:occprobnear}, the probability that the process within the first $n \lambda$ iterations reaches a search point (as parent individual) with fitness distance more than $3n / \lambda$, is $o(1)$.
\end{lemma}

\begin{proof}
  Consider a run of the self-adjusting \oclea starting in the situation of Lemma~\ref{lem:occprobnear}. Denote by $X_t$ the fitness distance at time $t$. We start by bounding the negative drift $E(\max\{0,X_t-X_{t-1}\})$ of the $X$ process while it is at most $3n/\lambda$. If the current rate is~ $r$, then by Lemma~\ref{lem:noback} with probability at least $1 - \exp(-\tfrac 12 \lambda (1-o(1)) e^{-r/F})$  there is an individual that used rate $r/F$ and that did not flip any zero-bit into a one-bit. Let us call this event ``$A$'' and note that, naturally, under this event the drift 
  cannot be negative as the individual without flipped zeroes has at an least as good fitness as the parent. 
  
  We now analyze the case that $A$ does not hold. Consider an individual conditional on that it uses rate $r/F$ and at least one zero-bit was flipped into a one-bit. The number of such bad bits follows a distribution $(X \mid X \ge 1)$ with $X \sim \Bin(n-X_{t-1},r/Fn)$ and has expectation at most $1 + r/F$ by Lemma~\ref{lem:condbinomial}. For an individual using rate $rF$, the expected number of  bad flips is $(n-k) \frac{rF}{n} \le rF$. Consequently, noting that $1+r/F \le rF$ when $r \ge F$ and $F \ge \sqrt 2$, the expected number of bad flips in all individuals (conditional on not $A$) is at most $\lambda r F$ 
  and this is an upper bound on the negative drift.
  
  In summary, in an iteration starting with rate $r$, the negative drift is at most 
  \begin{equation}
  \lambda r F \exp(-\tfrac 12 \lambda (1-o(1)) e^{-r/F}).\label{eq:negdrift}
  \end{equation} 
  With Lemma~\ref{lem:occprobnear}, we can estimate the probability to have a certain rate. Hence the expected negative drift is
  \begin{align*}
  E(\max\{0,X_t-X_{t-1}\}) & \le \sum_{i=1}^{\log_F \rmax} \Pr(r = F^i) \lambda F^i F \exp(-\tfrac 12 \lambda (1-o(1)) e^{-F^i/F})\\
  & \le \sum_{i=2}^\infty \exp(-8 F^{i-1}) \lambda F^{i+1} \exp(-\tfrac 12 \lambda (1-o(1)) e^{-F^{i-1}}) \\
  & \quad + \lambda F^{2} \exp(-\tfrac 12 \lambda (1-o(1)) e^{-1}).
  \end{align*}
  Note that\footnote{In this part of the proof, we use the fact that $F=32$. This does not mean that for other not too small values of $F$ we would not obtain similar results, but it increases the readability to work with this concrete value.} for $i \ge \lceil \log_F(\ln \lambda) +1 - \frac 15 \rceil = i^*$, 
  we have $\lambda \le \exp(2 F^{i-1})$ and thus $\exp(-8 F^{i-1}) \lambda F^{i+1} = \exp(-(1-o(1))8 F^{i-1}) \lambda \le \exp(-(1-o(1)) 6 F^{i-1})$. Naturally, $\exp(-\tfrac 12 \lambda (1-o(1)) e^{-F^{i-1}}) \le 1$. Hence 
  \begin{align*}
  \sum_{i=i^*}^\infty \exp&(-8 F^{i-1}) \lambda F^{i+1} \exp(-\tfrac 12 \lambda (1-o(1)) e^{-F^{i-1}}) \le \sum_{i=i^*}^\infty \exp(-(1-o(1)) 6 F^{i-1}) \\
  &\le \exp(-(1-o(1)) 6 F^{i^*-1}) \le \lambda^{-3 (1-o(1))}.
  \end{align*}
  
  For $i < \log_F(\ln \lambda) +1 - \frac 15$, we have $\exp(-\tfrac 12 \lambda (1-o(1)) e^{-F^{i-1}}) \le \exp(-\tfrac 12 (1-o(1)) \lambda^{1/2})$   and $\exp(-8 F^{i-1}) \lambda F^{i+1} = O(\lambda)$. Hence
  \begin{align*} 
  \sum_{i=1}^{i^*-1} \exp&(-8 F^{i-1}) \lambda F^{i+1} \exp(-\tfrac 12 \lambda (1-o(1)) e^{-F^{i-1}}) \\
  &\le O(\log\log \lambda) O(\lambda) \exp(-\tfrac 12 (1-o(1)) \lambda^{1/2}) = o(\lambda^{-3}).\\
  \end{align*}
 Consequently, $E(\max\{0,X_t-X_{t-1}\}) \le \lambda^{-3(1-o(1))}$.
  
  Define inductively $Y_0 = 0$ and $Y_t = Y_{t-1} + \max\{0, X_t-X_{t-1}\}$, if $\max\{X_s \mid s \in [0..t-1]\} \le 3n/\lambda$ and $Y_t = Y_{t-1}$ otherwise. In other words, the $Y$ process collects all the moves of the $X$ process that go away from the optimum until the $X$ process goes above $3n/\lambda$. 
  
  By our above computation, we have $E(Y_t) \le t \lambda^{-3(1-o(1))}$. Consequently, by Markov's inequality, we have 
  \[\Pr(Y_t \ge t \lambda^{-2}) \le \lambda^{-1+ o(1)}\] 
  for all $t \in \N$. In particular, for $t = n \lambda$, we have $\Pr(Y_t \ge n/\lambda) \le \lambda^{-1+o(1)}$. Note that $Y_t \le n/\lambda$ implies $X_s \le 3n/\lambda$ for all $s \le t$.
\qed\end{proof}

\begin{lemma}\label{lem:nearmain}
  In the situation of Lemma~\ref{lem:occprobnear}, with probability at least $\frac 34$ there is a $T^* = O(n \ln(n/\lambda + 2) / \lambda)$ such that $k_{T^*} = 0$.
\end{lemma}

\begin{proof}
  Since we are proving an asymptotic statement, we can assume that $n$ is as large as we find convenient. Consider a run of the self-adjusting \oclea from our starting position. Let $T$ be the first time that the fitness distance is larger than $3n / \lambda$, if such a time exists, and $T = \infty$ otherwise. Let $k_t$ denote the fitness distance at time $t$ and $r_t$ the rate used in iteration $t$, if $t \le T$, and $(k_t,r_t) := (0,F)$ otherwise. We show that the process $(k_t,r_t)$ reaches $(0,F)$ in time $T^*$ with probability at least $1 - 1/e^2$.
  
  We use a two-dimensional drift argument. Let $\gamma = 2F$ and define $g : \N \times \N \to \R$ by $g(k,r) = k+\gamma (r-F)$ for all $k$ and $r$. We show that if for some $t$ we have $(k,r) = (k_t,r_t)$, then $(k',r'):=(k_{t+1},r_{r+1})$ satisfies 
  \begin{equation}
  E(g(k',r')) \le g(k,r) (1 - \tfrac{\lambda}{10n})\label{eq:multidrift}
  \end{equation} 
  when assuming $n$ to be sufficiently large.
  
  There is nothing to show in the artificial case when $k > 3n/\lambda$ as we have, by definition, $g(k',r')=0$ in this case. Among the interesting cases, we consider first that $r=F$. We obtain an improvement in fitness in particular if there is an offspring that uses rate $r/F=1$, flips exactly one of the $k$ missing bits, and flips no other bit. Hence the probability to make a positive fitness progress is at least 
%  \begin{align*}
%  1 - (1 - \tfrac 12 (1 - \tfrac 1n)^{n-1} \tfrac kn)^\lambda
%  &\ge 1 - (1 - \tfrac k {2en})^\lambda \\
%  &\ge 1 - \exp(-\tfrac {k\lambda}{2en} )\\
%  &\ge \tfrac 35 (1-o(1)) \tfrac {k\lambda}{2en},  
%  \end{align*}
  \begin{align*}
  1 - (1 - \tfrac 12 (1 - \tfrac 1n)^{n-1} \tfrac kn)^\lambda
  \ge 1 - (1 - \tfrac k {2en})^\lambda 
  \ge 1 - \exp(-\tfrac {k\lambda}{2en} )
  \ge \tfrac {k\lambda}{3en},  
  \end{align*}
  %where we used $(1-\frac 1n)^{n-1} \ge \frac 1e$ and Lemma~\ref{lestimate}~\ref{itestimate2}. 
  where we used $(1-\frac 1n)^{n-1} \ge \frac 1e$, $\frac{k\lambda}{2en}\le \frac{3}{2e}<\frac{3}{2}\cdot\frac{1}{2}$ and Lemma~\ref{lestimate}~\ref{itestimate2}. 
  The expected negative progress is at most $\lambda F^2 \exp(- (1+o(1)) \tfrac 1{2e} \lambda)$ as shown in~\eqref{eq:negdrift}. This negative drift can be assumed to be $O(n^{-2})$ by taking the implicit constant in the assumption $\lambda = \Omega(\log n)$ large enough. 
  %Consequently, $E(k') \le k - (1-o(1)) \tfrac 3{10e}  \tfrac {\lambda k}n$. 
  Consequently, $E(k') \le k - \tfrac 1{3e}  \tfrac {\lambda k}n + O(n^{-2})$. 
  
  Regarding $r'$, we note that by Lemma~\ref{lem:drift-rate-near} we have $\Pr(r' = F^2) \le (1+o(1)) \frac{\lambda k}{n} F^2 \exp(-F^2)$ and $r'=F$ otherwise. Hence $E(r') = F + (1+o(1)) (F-1) F^3 \frac{\lambda k}{n} \exp(-F^2)$. Consequently, 
  \begin{align*}
E(g(k,r) - g(k',r')) &\ge  \tfrac 1{3e} \tfrac {\lambda k}n - O(n^{-2}) - \gamma (1+o(1))(F-1) F^3 \tfrac{\lambda k}{n} \exp(-F^2) \\
&= \tfrac{\lambda k}{n}(\tfrac 1{3e} - \gamma (F-1) F^3 \exp(-F^2) - o(1)) \\
&\ge \tfrac{\lambda k}{n} \tfrac 1{10} = g(k,r) \tfrac{\lambda}{10n}. %needs that F is not too small
\end{align*}  
  Let now be $r > F$. Note that the minimum fitness loss among the offspring is at most the minimum number of bits flipped, which in expectation is at most the number of bits flipped in the first offspring, which is exactly $Fr$. Consequently, we have $E(k') \le k+Fr$.
  For $r'$, we note that by Lemma~\ref{lem:drift-rate-near}, we have $r' = Fr$ with probability at most $\exp(-9r)$ and we have $r' = \frac{r}{F}$ otherwise. Consequently, $E(r') \le Fr \exp(-9r) + \frac{r}{F}$. This yields 
  \begin{align*}
  E(g(k,r) - g(k',r')) &\ge -Fr + \gamma(r - Fr \exp(-9r) - \tfrac{r}{F}) \\
  &\ge r(-F+\gamma - F \exp(-9F^2) - \tfrac 1F) \\
  &\ge r(-F+\gamma - \tfrac 2F) \ge 31r = r + 30r \ge F^2 + 30r  \\
  &\ge 32^2 + 30r \ge \tfrac{\lambda}{10n} (k+\gamma r) \ge g(k,r) \tfrac{\lambda}{10n}, 
  \end{align*}
	where we used that $\lambda \le 2n$; note that $\lambda > 2n$ gives $k_0 =0$. 
	
  We have thus shown~\eqref{eq:multidrift} for all $(k,r)$. Since we start the process with a $g$-potential of at most $g(2n/\lambda,F) = \frac{2n}{\lambda}$, the multiplicative drift theorem with tail bounds~\cite[Theorem~5]{DoerrG13algo} gives that after $t = \lceil \frac{10n}{\lambda}(2+\ln(\frac{2n}{\lambda})) \rceil$ iterations, we have $\Pr(g(k_t,r_t) > 0) \le \frac 1{e^2}$. Consequently, with probability $1-\frac 1{e^2}$, the potential is zero at time $t$, which implies $k_t = 0$ or $k_t > \frac{3n}{\lambda}$. By Lemma~\ref{lem:goback}, note that $\lambda = \Omega(\log n)$ implies $t = O(n) = o(n\lambda)$, the probability that $k_t > \frac{3n}{\lambda}$ is $o(1)$, hence with probability at least $\frac 34$, we have indeed $k_t = 0$.
\qed\end{proof}

\begin{theorem}\label{thm:expected-time-near}
	Assume $k_0\le \frac{2n}{\lambda}$ and $r_0 \le \frac 79 \ln \lambda$. Then there is a $t=O(n\ln(n/\lambda+2)/\lambda)$ such that with probability at least~$\frac 12$, we have $k_{t} = 0$.
\end{theorem}

\begin{proof}
	We first show that with good probability we quickly reach the initial situation of Lemma~\ref{lem:nearmain}. The probability of observing $R^* - 1 \coloneqq \log_F(r_0) -1 \le \log_F(\frac 79 \ln \lambda)-1$ 
	rate-decreasing steps in a row by Lemma~\ref{lem:drift-rate-near} is at least
	\[
	\prod_{i=2}^{R^*} \left(1-\exp(-9F^i)\right) 
	\ge 1-\sum_{i=2}^{R^*}  \exp(-9F^i) \ge 1 - 0.001
	\]
	by the Weierstrass product inequality (Lemma~\ref{lestimate}~\ref{itestimate3}). 
	
	The probability of not flipping any zero-bits 
	in at least one offspring, resulting in not increasing fitness distance, 
	is for rate $r \le \frac 79 \ln\lambda$ at least $1 - \exp(-\tfrac 12 \lambda (1-o(1)) e^{-r/F})$ by Lemma~\ref{lem:noback}.
	By a union bound over $R^*-1$ iterations, the probability 
	of decreasing the initial rate to~$F$ in $O(\log \log \lambda)$ iterations without losing fitness is at least $1 - 0.001 -(R^*-1) \exp(-\tfrac 12 \lambda (1-o(1)) e^{-r/F}) \ge 5/6$ 
	for sufficiently large~$n$.
	
	We can now apply Lemma~\ref{lem:nearmain} and obtain that with probability at least $\frac 34$ we have found the optimum within $t=O(n\ln(n/\lambda+2)/\lambda)$ iterations. 
	This show the claim.
	\qed\end{proof}

\subsection{Proof of Theorem~\ref{theo:main}}

From the separate analyses of the two regimes above, we can now easily derive our main result.
\begin{proof}
Starting with arbitrary initialization, Theorem~\ref{theo:runtime-far} along with a Markov bound yield 
that with probability $\Omega(1)$ after $t=O(n/\log \lambda)$ iterations 
a search point is reached such that $k_t\le 2n/\lambda$ and $r_t < 0.6(\ln\lambda)$. 
Assuming this to happen, the assumptions of Theorem~\ref{thm:expected-time-near} are satisfied. 
Hence, after another $O((n\log n)/\lambda)$ iterations the optimum is found with probability 
at least~$1/2$. Altogether, with probability $\Omega(1)$ the optimum is found from an arbitrary initial 
\om-value and rate within 
$T^*= O(n/\log \lambda + (n\log n)/\lambda)$ iterations. The claimed expected time now follows by a standard 
restart argument, more precisely by observing that 
after expected $O(1)$ repetitions of a phase  of length $T^*$ the optimum is found.
\qed\end{proof}

\section{Experiments}

To gain some insight that cannot be derived from our asymptotic analysis, we performed a few numerical experiments. To this end we implemented the \oclea in C++11 using the \emph{default random engine} to generate pseudo-random numbers. The runtime is still measured via the number of generations until optimum is found.

We first see in Figure \ref{figure:one_run} how fitness distance and mutation strength evolve in one run for $n=100$, $\lambda=12$ and $F=1.2$. We used this small value of $n$ to increase the readability of the figure, we used larger values for $n$ in the remainder. Given the small value of $n$, we used a small mutation update factor of $1.2$ instead of the value $F=32$ used in our theoretical analysis. This run uses Algorithm \ref{alg:onelambda} with $r^{\init}=F$. We see that the algorithm prefers large mutation strengths at the beginning and small mutation strengths near the end of the optimization process. We also see that fitness distance can increase occasionally, in particular, when the rate is higher (in the plot, this happened in iteration 52 and iteration 88).

In Figure \ref{figure:runtime_cmp}, we display the average runtime over 100 runs of different versions of the \oclea on \onemax for $n=10^5$ and $\lambda=100,200,\dots,1000$. For our self-adaptive \oclea (Algorithm \ref{alg:onelambda}), we used the update strengths $F \in \{1.2, 2, 32\}$. We did experiments also for $F = 1.05$, but the results were clearly inferior, so to not overload this figure we do not visualize them. We always set the initial mutation strength to $r^{\init}=F$. We further regard the classic \oclea using a static mutation rate of $\tfrac 1n$ and the \oclea with fitness-dependent mutation rate $p=\max\{\frac{\ln\lambda}{n\ln(en/d)},\frac{1}{n}\}$ as presented in \cite{BadkobehLS14}. 

The results clearly show that the update factor of $F=32$ used in our mathematical analysis gives sub-optimal results for these values of $\lambda$ and $n$. Recalling the working principle of the self-adaptive \oclea, this is not overly surprising. Even using the minimal possible rate $r=F$, the algorithm creates half of the offspring using an incredible large mutation probability of $F^2/n = 1024/n$. It is quite clear that this cannot be overly effective, but this can also be seen from the figure. The runtime of the self-adaptive \oclea with $F=32$ is very close to the runtime of the static \oclea for half the $\lambda$-value, suggesting that half the offspring created by the self-adaptive \oclea, most likely the ones created with a mutation rate of $F^2/n$, had no impact on the process. 

The results in Figure~\ref{figure:runtime_cmp} also show that the fitness-dependent mutation strength of~\cite{BadkobehLS14} leads to a very good performance. In principle, of course, it is clear that the best fitness-dependent rate gives better results than any self-regulating rate since the latter needs to use also sub-optimal rates to find out what is the best rate. That the rate suggested in~\cite{BadkobehLS14}, a paper mostly concerned with asymptotic runtimes, shows such good results, is remarkable.

To ease the comparison of the algorithms having a similar performance, we plot in Figure \ref{figure:relative_runtime_cmp} these runtimes relative to the one of the classic \oclea. This shows that in most cases, the EAs using a dynamic mutation rate outperform the classic \oclea. We also notice that the self-adaptive EA appears to outperform the one using the fitness-dependent rate for sufficiently large values of $\lambda$, e.g., for $\lambda \ge 200$ when $F=1.2$.

To understand how our tie-breaking rule influences the performance, we also ran the self-adapting \oclea without the bias towards smaller rates when breaking ties. In Figure \ref{figure:breaking_ties}, we again plot the average runtimes over 100 runs relative to the results of static \oclea. We use the three update factors $1.2$, $2$, and $32$ and the two tie-breaking rule of preferring the smaller rate in case of ties (as in our theoretical analysis) and random tie-breaking, that is, choosing uniformly at random an offspring with maximal fitness and taking its rate as the new rate of the algorithm. While for the two larger factors $F=2$ and $F=32$, no significant differences are visible, we see that for $F=1.2$ random tie-breaking surpasses biased tie-breaking significantly when $\lambda$ become larger than $200$. 

To understand how the tie-breaking rule influences the mutation strength chosen by the algorithm, we plot in Figure~\ref{figure:rate} the mutation strength used at each fitness distance with a setting of $n=10000$, $\lambda=500$, and $F=1.2$. We regarded one exemplary runs of our algorithm with each tie breaking rule. In each of these two experiments, we determined the set of all pairs $(d_t,r_t)$ such that in iteration $t$, the fitness distance of the parent individual was $d_t$ and its rate was $r_t$. We then plotted these sets, where to increase the readability we connected the points to polygonal curve. This visualization clearly shows that random tie breaking lets the algorithm pick larger rates more frequently. Together with the better runtimes, it appears that biased tie-breaking has a small negative effect on the choice of the mutation strength.

Finally, we regard the question of how to set the initial rate $r^{\init}$. From the general experience that larger mutation rates are more profitable at the start of the search process, one could guess that it is a good idea to start with the largest possible rate $\rmax = F^{\lfloor \log_F (n/(2F)) \rfloor}$ instead of the smallest possible rate $r_{\min}=F$. For the settings used in Figure~\ref{figure:rate}, that is, $n = 10000$, $\lambda = 500$, and $F = 1.2$, we obtain (as average of 100 runs) the runtimes given in Table~\ref{table:init}. So indeed an initialization with a larger rate gives some improvement. Since it might be a particularity of the \onemax test function that huge rates are initially beneficial, we would not give out a general recommendation to start with the rate $\rmax$, but only state that we observed moderate performance differences from using different initial rates, making the initial rate not the most critical parameter of the algorithm, but still one that can be worth optimizing.

\begin{table}
	\centering
	\begin{tabular}{|c|c|c|}
		\hline
		Average runtime & Biased ties breaking & Random ties breaking\\
		\hline
		$r^{\init}=r_{\min}$ & 2137 & 2011\\
		\hline
		$r^{\init}=r_{\max}$ & 2080 & 1974\\
		\hline
	\end{tabular}
	\caption{Comparison of the average runtime of 100 runs for different initial mutation rates ($n=10000$, $\lambda=500$, and $F=1.2$) }
	\label{table:init}
\end{table}

\begin{figure}
	\includegraphics[scale=0.90]{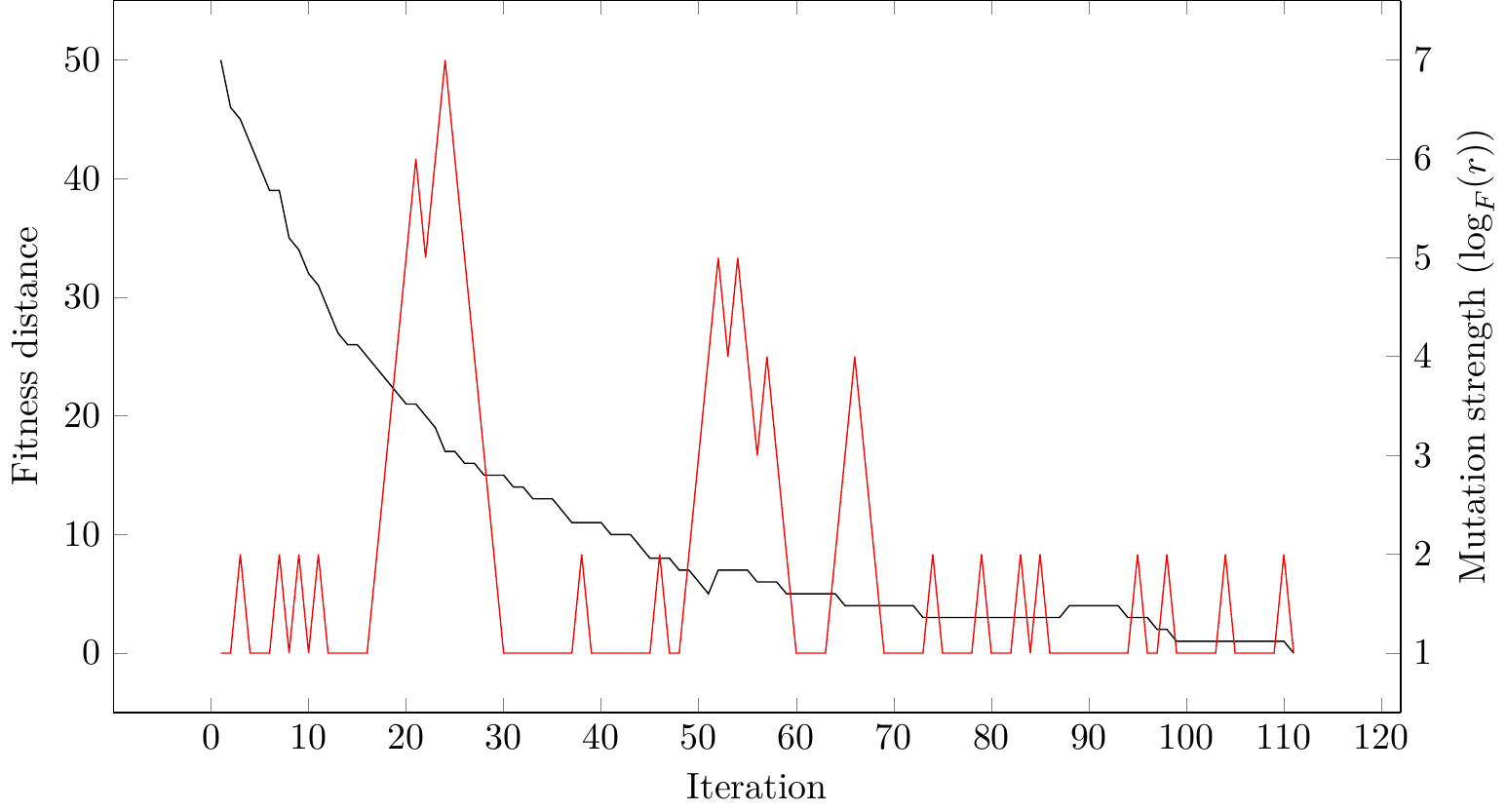}
	\caption{Development of fitness distance and mutation strength in one run of self-adapting \oclea on \onemax ($n=100$, $F=1.2$, $\lambda=12$)}
	\label{figure:one_run}
\end{figure}

\begin{figure}
	\centering
	\includegraphics[scale=0.90]{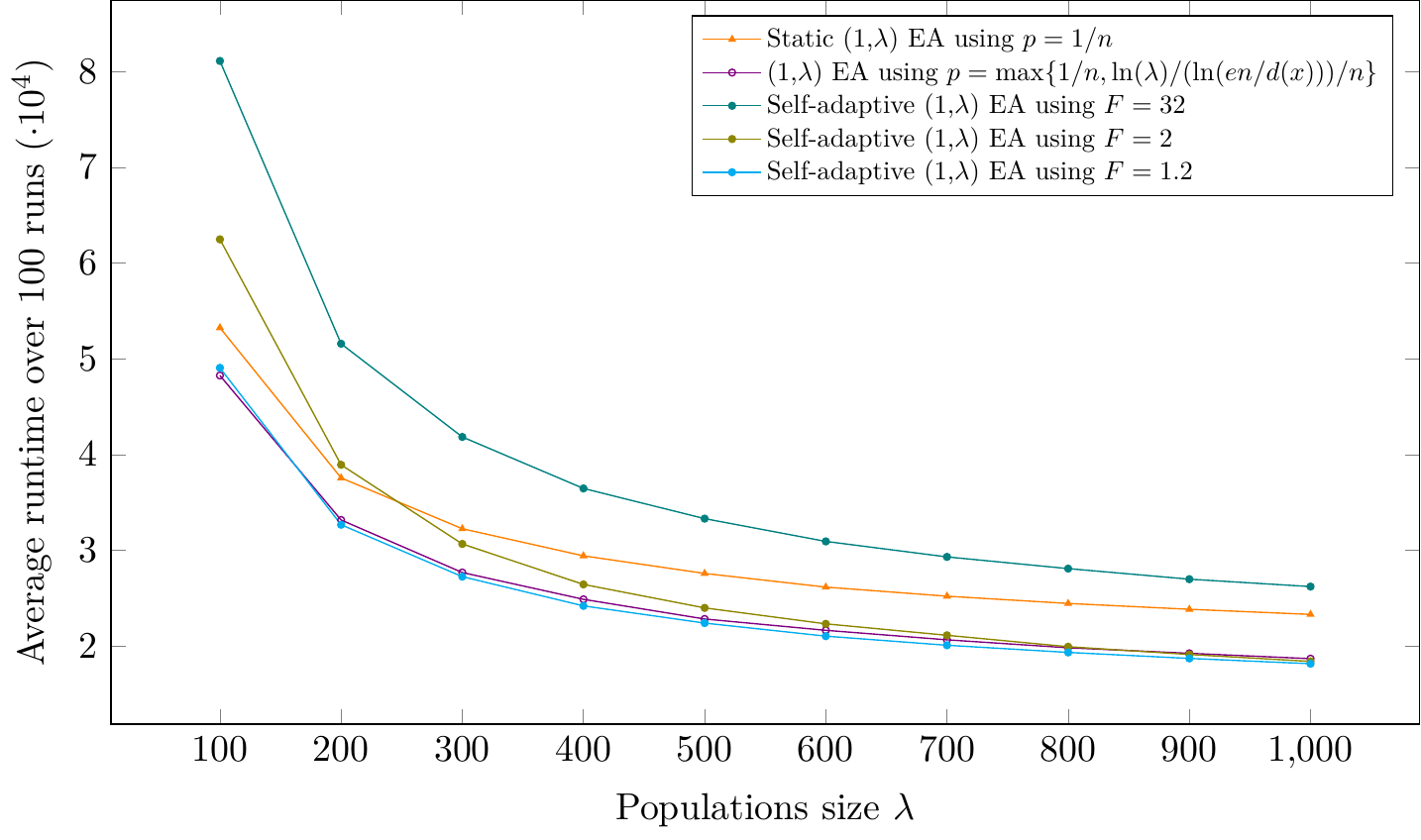}
	\caption{Average runtime over 100 runs of five variants of the \oclea  on \onemax for $n=10^5$.}
	\label{figure:runtime_cmp}
\end{figure}

\begin{figure}
	\centering
	\includegraphics[scale=0.90]{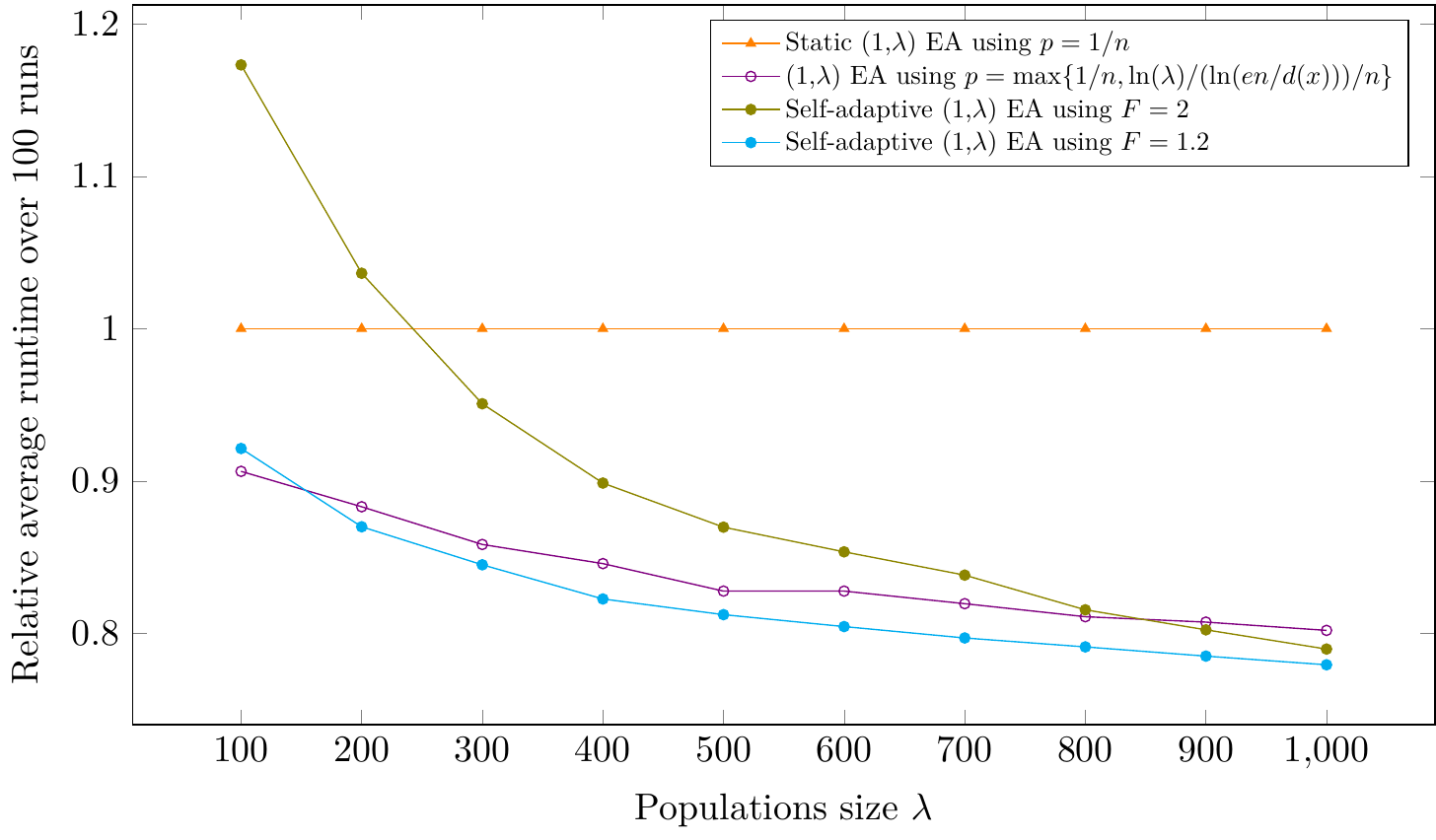}
	\caption{Average runtime of three dynamic \oclea{}s relative to the average runtime of the static \oclea on \onemax($n=10^5$)}
	\label{figure:relative_runtime_cmp}
\end{figure}

\begin{figure}
	\centering
	\includegraphics[scale=0.90]{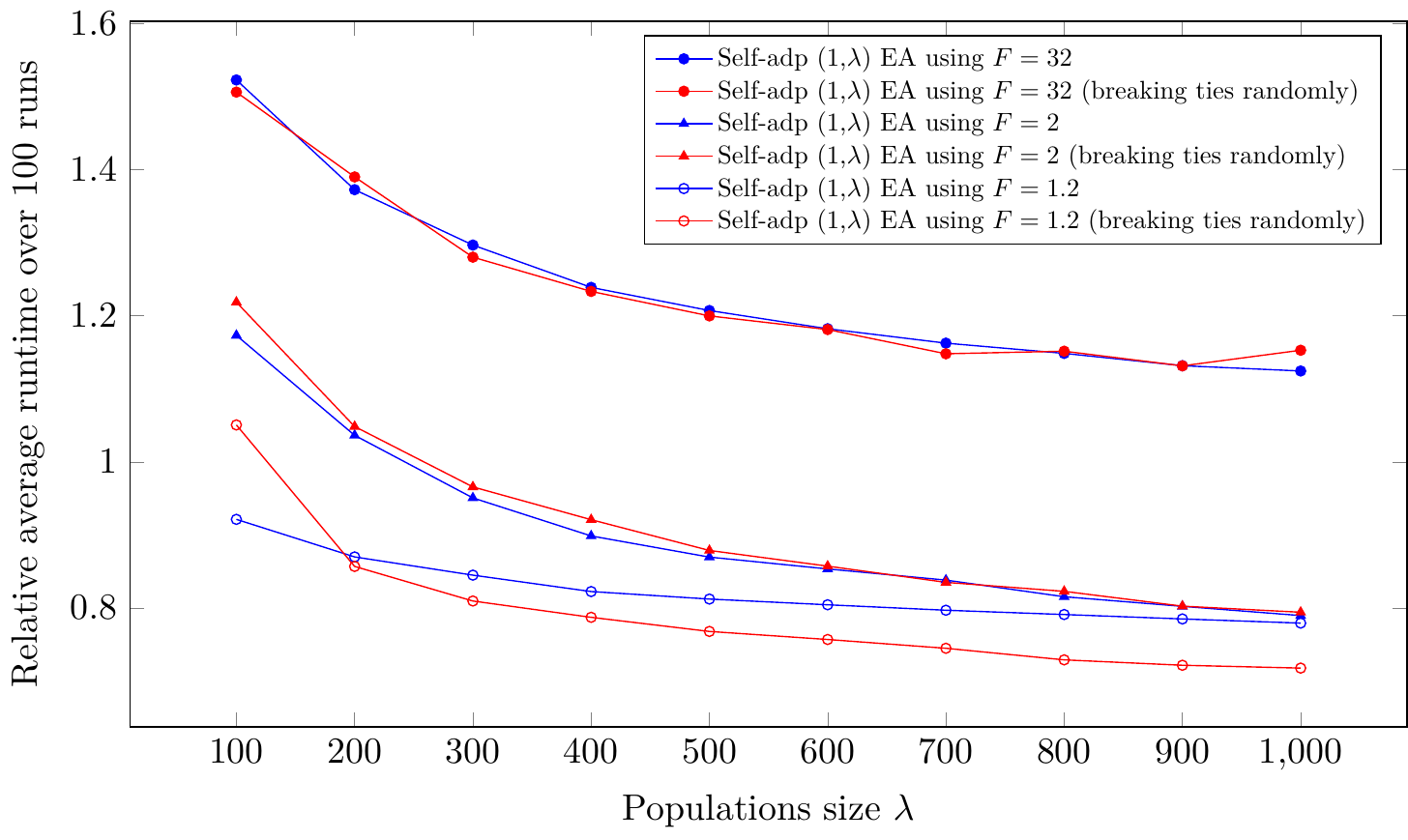}
	\caption{Relative average runtime of self-adapting \oclea{}s  with different tie breaking rules on \onemax ($n=10^5$)}
	\label{figure:breaking_ties}
\end{figure}

\begin{figure}
	\centering\includegraphics[scale=0.90]{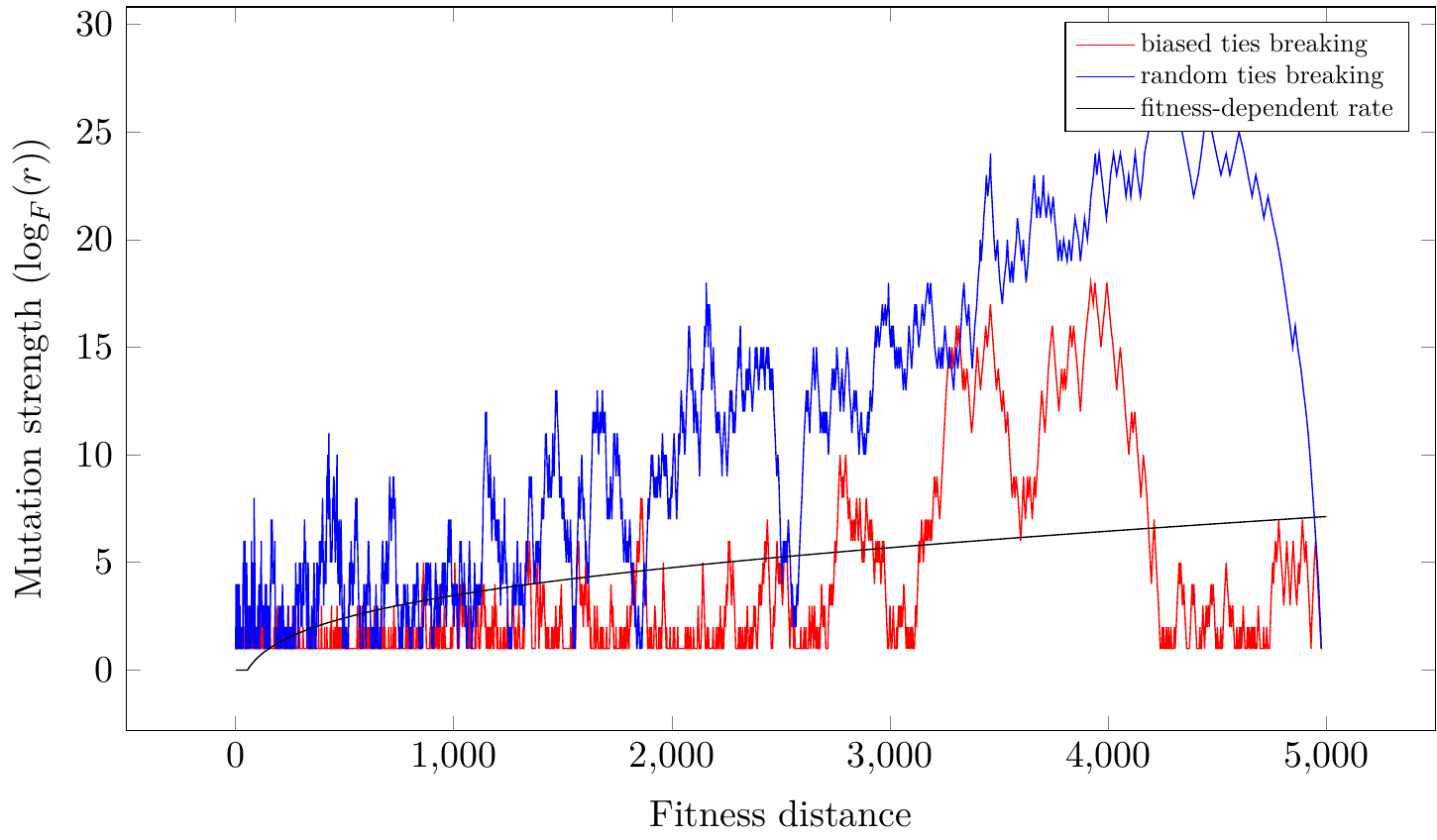}
	\caption{Mutation strengths used at a certain fitness distance level in two example runs of the self-adapting \oclea on \onemax ($n=10000$, $\lambda=500$, $F=1.2$). For comparison, also the fitness-dependent rate proposed in~\cite{BadkobehLS14} is plotted. Recall that a mutation strength of $r$ in the self-adapting runs means that in average half the offspring use the rate $r/F$ and half use the rate $rF$.}
	\label{figure:rate}
\end{figure}

\section*{Conclusions}

In this work, we have designed and analyzed a self-adaptive \oclea using a  
simple scheme for mutating the mutation rate. We have proven that for $\lambda = \Omega(\log n)$ it achieves an expected runtime (number of fitness evaluations) of $O(n\lambda/\log\lambda+n\log n)$ on \om, 
which is best-possible for $\lambda$-parallel mutation-based unbiased black-box algorithms. Hence, 
we have identified a simple and natural example where self-adaptation of strategy parameters 
in discrete EAs can lead to provably optimal runtimes that beat all known static parameter settings. 
Moreover, a 
relatively complicated and partly unintuitive self-adjusting scheme for the mutation rate proposed 
in~\cite{DoerrGWY18} can be replaced by our simple endogenous scheme.

The analysis of this \oclea has revealed a non-trivial stochastic process in the cross product of fitness distances and mutation rates. 
We have advanced the techniques for the analysis of such two-dimensional processes, both via two new lemmas on occupation probabilities and by proposing suitable potential functions allowing to use classic drift theorems. 

Altogether, we are optimistic that 
our research helps pave the ground for further uses and analyses of self-adaptive EAs.

\paragraph{Acknowledgments}

The authors thank Christian Gie{\ss}en for useful discussions on this topic. 
This work was supported by a public grant as part of the Investissement d'avenir project, reference ANR-11-LABX-0056-LMH, LabEx LMH, in a joint call with Gaspard Monge Program for optimization, operations research and their interactions with data sciences. 
This publication is based upon work from COST Action CA15140, supported by COST.  

\bibliographystyle{alpha}

\newcommand{\etalchar}[1]{$^{#1}$}

%\bigskip
%
%\bibliography{selfadaptive}
}%end sloppy
\end{document}